\pdfoutput=1
\documentclass[letterpaper, 10 pt, journal, twoside]{IEEEtran}

\usepackage{graphics} 
\usepackage{subfig} 
\usepackage{wrapfig}

\usepackage{amsthm} 
\usepackage{amsmath,amssymb}
\usepackage{graphicx}
\usepackage{tabularx}
\usepackage{multirow, makecell}
\usepackage{diagbox}
\usepackage{slashbox}
\usepackage{rotating}
\usepackage{cite}

\usepackage{url}
\usepackage{bm}
\usepackage[table]{xcolor}
\usepackage{hyperref}
\usepackage{siunitx} 
\usepackage{booktabs}
\usepackage{cleveref}
\usepackage{mathtools} 
\usepackage{upgreek} 
\usepackage{lipsum} 

\usepackage{enumitem} 

\usepackage[ruled,vlined]{algorithm2e}

\usepackage{pifont} 
\newcommand{\xmark}{\ding{55}}%

\usepackage[nolist, nohyperlinks]{acronym}
\acrodef{MPC}[MPC]{Model Predictive Control}
\acrodef{QP}[QP]{Quadratic Program}
\acrodef{CBF}[CBF]{Control Barrier Function}

\newcommand{\vx}{{\boldsymbol x}}
\newcommand{\vu}{{\boldsymbol u}}
\newcommand{\vq}{{\boldsymbol q}}

\newcommand{\vp}{{\boldsymbol p}} 
\newcommand{\vv}{{\boldsymbol v}} 

\newcommand{\lieder}{L}
\newcommand{\StateSpace}{\mathcal{X}}
\newcommand{\ControlSpace}{\mathcal{U}}
\newcommand{\RealSpace}{\mathbb{R}}

\newcommand{\Rn}{\mathbb{R}^{n}}
\newcommand{\Rm}{\mathbb{R}^{m}}
\newcommand{\Rnm}{\mathbb{R}^{n \times m}}

\newcommand{\calC}{\mathcal{C}} 
\newcommand{\calK}{\mathcal{K}} 
\newcommand{\calR}{\mathcal{R}} 

\newcommand{\calP}{\mathcal{P}} 
\newcommand{\calW}{\mathcal{W}} 
\newcommand{\calB}{\mathcal{B}}
\newcommand{\calO}{\mathcal{O}} 
\newcommand{\calS}{\mathcal{S}} 
\newcommand{\calN}{\mathcal{N}} 
\newcommand{\calA}{\mathcal{A}} 

\newcommand{\calJ}{\mathcal{J}} 
\newcommand{\calH}{\mathcal{H}}
\newcommand{\calU}{\mathcal{U}}
\newcommand{\calF}{\mathcal{F}}
\newcommand{\va}{{\boldsymbol a}}

\DeclareMathOperator*{\argmax}{arg\,max}
\DeclareMathOperator*{\argmin}{arg\,min}

\newtheorem{definition}{Definition}
\newtheorem{theorem}{Theorem}
\newtheorem{lemma}{Lemma}

\theoremstyle{definition}
\newtheorem{remark}{Remark} 

\theoremstyle{definition}

\theoremstyle{definition}

\theoremstyle{definition}

\theoremstyle{definition}
\newtheorem{assumption}{Assumption}

\DeclareCaptionFont{mysize}{\fontsize{8}{9.6}\selectfont}
\title{\texttt{SEAMLiS}: Visibility-Aware Safety for Perception-Limited Multi-Robot Exploration}

\author{Taekyung Kim$^{1}$, Rahul H Kumar$^{1}$, Aswin D. Menon$^{2}$, Tzu-Hsiang Lin$^{3}$, Dimitra Panagou$^{1,2}$
\thanks{$^{1}$Department of Robotics, $^{2}$Department of Aerospace Engineering, $^{3}$Department of Mechanical Engineering, University of Michigan, Ann Arbor, MI, 48109, USA {\tt\footnotesize \{taekyung, rahulhk, admenon, szlin, dpanagou\}@umich.edu} }%
}

\begin{document}
\maketitle

\begin{abstract}
Autonomous exploration in unknown environments is typically driven by informative frontiers, viewpoints, or trajectories, while local safety controllers avoid obstacles represented in the current map. Under finite sensing range and limited field of view, this separation can be unsafe: an exploration stack may plan optimistically through unobserved space and steer the sensor toward information gain rather than along the direction of motion, causing hidden obstacles to be detected too late for bounded-actuation avoidance. This paper presents \texttt{SEAMLiS} (\textbf{S}afe \textbf{E}xploration for \textbf{A}utonomous \textbf{M}ulti-Robot Systems Under \textbf{Li}mited \textbf{S}ensing), a modular execution-layer safety framework for decentralized multi-robot exploration. \texttt{SEAMLiS} preserves the upstream exploration stack, including the goal allocator and local planner, and enforces safety at the execution layer through perception-aware attitude and positional filters. A \texttt{gatekeeper}-based attitude filter switches between a visibility-promoting yaw policy and a velocity-tracking backup policy to preserve visibility of the critical known-free/unknown boundary with sufficient braking margin. A \ac{CBF}-based positional filter then avoids known obstacles, newly detected obstacles, and other robots. We provide sufficient collision-avoidance conditions and validate the framework in randomized simulation, Isaac Sim, and Crazyflie hardware experiments. Results show collision-free exploration across tested single- and multi-robot settings while retaining much of the efficiency of visibility-promoting yaw control. \href{https://github.com/tkkim-robot/seamlis}{\textcolor{red}{[Code]}}  \href{https://www.taekyung.me/seamlis}{\textcolor{red}{[Project Page]}}\footnote{Project page: \href{https://www.taekyung.me/seamlis}{https://www.taekyung.me/seamlis}} \href{https://youtu.be/0EzbbFIb2fY}{\textcolor{red}{[Video]}} 
\end{abstract}

\section{INTRODUCTION}

Exploration and mapping of unknown environments are fundamental capabilities for autonomous robots operating in unstructured workspaces. The objective is to incrementally construct a map while navigating safely, typically through a hierarchical autonomy stack consisting of localization and mapping, global exploration-goal selection, local planning, and feedback control. A seminal work~\cite{yamauchi_frontierbased_1997} introduced \emph{frontiers}, defined as boundaries between known-free and unknown regions, and established frontier-driven exploration as a canonical paradigm. Since then, extensive literature has sought to reduce exploration time by improving frontier selection and viewpoint planning, including heuristic frontier-selection methods~\cite{dai_fast_2020,kim_multirobot_2024,xu_heuristicbased_2024}, next-best-view strategies~\cite{lindqvist_treebased_2024}, and learning-assisted and learning-based approaches~\cite{hu_voronoibased_2020,chaplot_learning_2020,li_deep_2020,ramakrishnan_poni_2022,li_learningaugmented_2023,cao_deep_2024}.

\begin{figure}[t]
    \centering
    \includegraphics[width=0.99\linewidth]{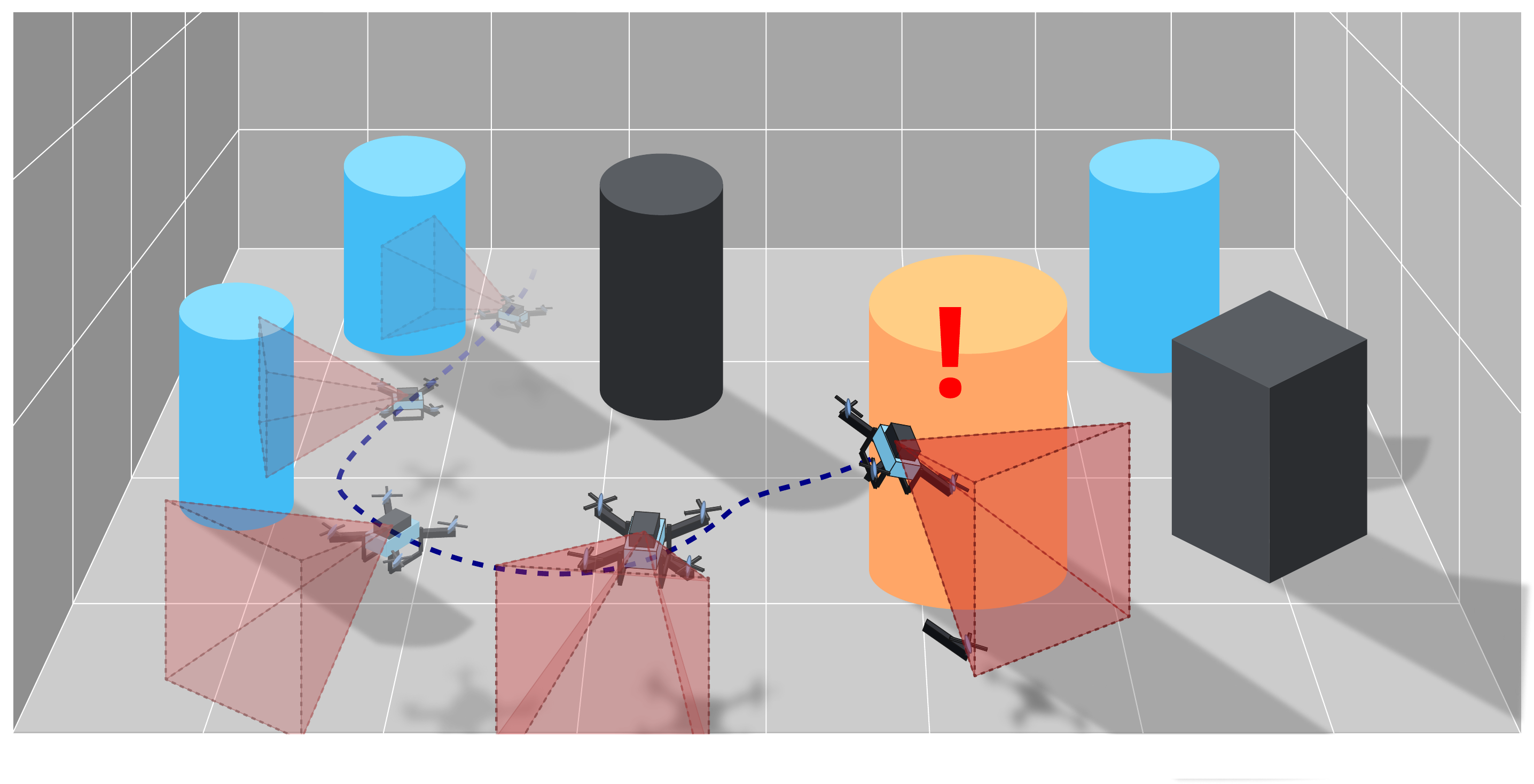}
    \caption{Motivating failure example for perception-limited robot exploration. A robot exploring an unknown environment steers its finite-FoV sensor toward information-rich regions, while its translational motion approaches an unseen obstacle. Because the direction of motion is not kept sufficiently visible, the hidden obstacle may be detected only after the remaining stopping distance has become insufficient, leading to a collision.}
    \label{fig:drone_problem}
\end{figure}

\begin{figure*}[t]
    \centering
    \includegraphics[width=0.92\textwidth]{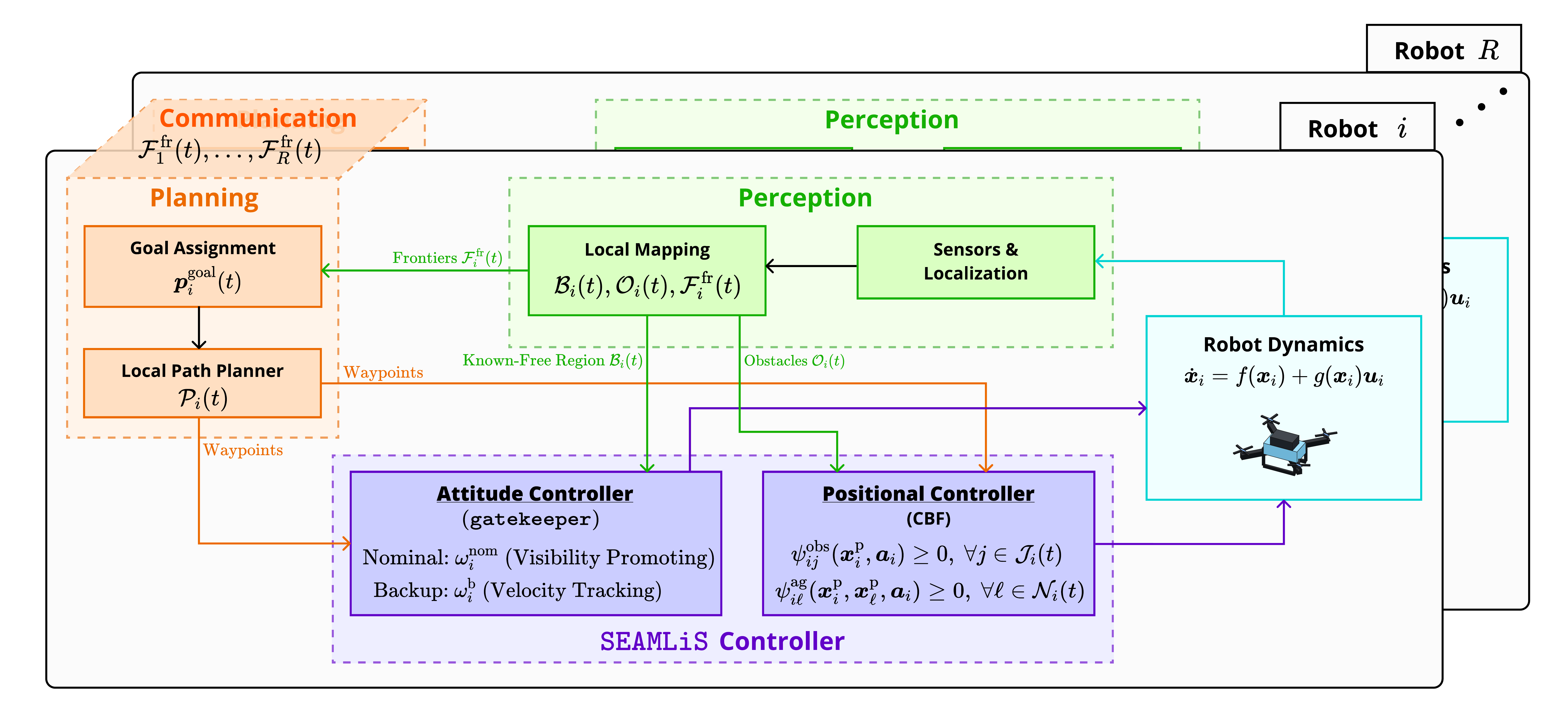}
    \caption{
    \texttt{SEAMLiS} integration in a decentralized exploration stack for $R$ robots.
    For robot $i$, the planning module assigns an exploration goal $\vp^{\textup{goal}}_{i} (t)$ from local or communicated frontier information $\calF^{\textup{fr}}(t) = \{ \calF^{\textup{fr}}_{1}(t), \ldots, \calF^{\textup{fr}}_{R}(t)\}$ and generates a waypoint path $\mathcal{P}_i(t)$. The perception module updates the local map, including the accumulated known-free set $\mathcal{B}_i(t)$, the local obstacle set $\mathcal{O}_i(t)$, and the frontier set $\mathcal{F}^{\textup{fr}}_i(t)$. \texttt{SEAMLiS} is placed at the execution layer: the attitude controller uses \texttt{gatekeeper} to filter a nominal visibility-promoting yaw command with a velocity-tracking backup policy, while the positional controller enforces CBF constraints for known obstacles, newly detected obstacles, and neighboring robots. The resulting control input is applied to the robot dynamics, and sensing and localization close the feedback loop.
    }
    \label{fig:seamlis_arch}
\end{figure*}

Multi-robot exploration can further improve coverage speed by distributing the mapping task across multiple platforms. Practical deployments, however, often face severe communication constraints: sharing full occupancy grids or dense maps can be bandwidth-intensive and brittle under packet loss. Consequently, many multi-robot systems are designed to be \emph{decentralized}, where robots share only limited information such as poses, topological summaries, or selected frontier candidates rather than complete maps~\cite{atanasov_decentralized_2015, zhou_racer_2023}. In this setting, a robot may be assigned a goal that lies on the boundary of, or even outside, its locally sensed region. The robot must then navigate through partially known space while relying on onboard sensing with finite range and limited field of view~(FoV).

A critical, yet often under-emphasized, issue is that efficiency-oriented exploration stacks can be unsafe under limited sensing. A common optimistic modeling choice is to treat unobserved space as collision-free for planning and control~\cite{gao_online_2016,zhong_informationdriven_2022,bena_safetyaware_2023}. This assumption is useful for maintaining exploration progress, but it becomes problematic when (i) the robot has a restricted FoV and finite sensing range, and (ii) the robot's yaw, or sensor heading, is controlled to maximize information gain rather than to observe the direction of motion. In holonomic platforms, or quadrotor-like models with decoupled yaw, a visibility-promoting yaw policy can induce \emph{sideways motion}: the robot translates along a direction that is not covered by its current FoV. Hidden obstacles along the direction of travel may then be detected only when the remaining stopping distance is insufficient. This failure mode can persist even when a local safety-critical controller, such as a \acf{CBF}-based \acf{QP}, is used for obstacle avoidance, because the safety filter may become infeasible once an obstacle is detected too late under bounded actuation~\cite{kim_visibilityaware_2025}. Fig.~\ref{fig:drone_problem} illustrates this motivating failure mode.

This paper addresses this gap by developing \texttt{SEAMLiS} (\textbf{S}afe \textbf{E}xploration for \textbf{A}utonomous \textbf{M}ulti-Robot Systems Under \textbf{Li}mited \textbf{S}ensing), a modular safety layer for decentralized multi-robot exploration. Rather than replacing the upstream goal assignment or planning modules, \texttt{SEAMLiS} operates at the \emph{execution layer}. It is designed to be plug-and-play with respect to both the global goal allocator and the local planner; any algorithm that outputs per-robot goals, including frontier-based methods~\cite{yamauchi_frontierbased_1997}, decentralized optimization-based methods, or learning-based exploration policies, can be used. The key idea is to ensure that each robot observes the critical boundary between its known-free region and the unknown region early enough to preserve the feasibility of collision avoidance under bounded acceleration. This perspective is aligned with recent analyses of perception-limited safety, which show that safety guarantees can fail when planners or controllers are agnostic to FoV and sensing-range constraints~\cite{firoozi_oampc_2025, kim_visibilityaware_2025}. The resulting modular architecture is illustrated in Fig.~\ref{fig:seamlis_arch}.

Accordingly, the central question in this paper \textbf{\emph{is not}} how to choose the most informative frontier, viewpoint, or trajectory, but \textbf{\emph{how to make an arbitrary upstream exploration stack safe}} when sensing is finite-range and FoV-limited, and when sensor heading is decoupled from translational motion.

The main contributions of this paper are:
\begin{itemize}
    \item We formalize decentralized multi-robot exploration under limited sensing and identify a safety gap overlooked by efficiency-oriented exploration stacks: optimistic planning combined with visibility-maximizing yaw control can delay hidden-obstacle detection until safe avoidance is infeasible.

    \item We propose \texttt{SEAMLiS}, a modular execution-layer safety framework that augments arbitrary, possibly asynchronous, multi-robot goal assignment and planning modules. This design preserves the scalability and efficiency of existing decentralized exploration pipelines while adding perception-aware safety enforcement under limited communication.

    \item We design a visibility-aware attitude safety filter using a \texttt{gatekeeper}-style mechanism~\cite{agrawal_gatekeeper_2024}. The filter permits nominal information-greedy yaw when certified safe and switches to a velocity-tracking backup policy when needed, ensuring that motion-relevant unknown boundaries remain observable with sufficient braking margin.

    \item We integrate the attitude safety filter with a \ac{CBF}-based positional controller for known obstacles, newly detected obstacles, and neighboring robots. We provide sufficient conditions under which the combined controller guarantees collision avoidance for all time despite initially unknown obstacles.

    \item We validate \texttt{SEAMLiS} in randomized simulation, Isaac Sim, and two-Crazyflie hardware experiments, showing that it eliminates the hidden-obstacle collisions observed in baseline attitude policies while retaining much of the efficiency of visibility-promoting exploration.
\end{itemize}

\section{RELATED WORK}
\label{sec:related_work}

Autonomous exploration has been studied from several complementary perspectives, including frontier and next-best-view~(NBV) planning, optimization-based multi-robot coordination, learning-based exploration, and safety-aware navigation under limited perception. The distinction most relevant to this work is not only how exploration goals are generated, but also whether the method is decentralized, whether finite range and limited field of view are explicitly modeled, and whether formal safety guarantees are provided. \autoref{tab:related_exploration} summarizes representative methods along these axes.

\begin{table*}[t]
\centering
\caption{Comparison with representative exploration and perception-limited safety methods. Here, ``Limited Sensing" indicates that the method explicitly models finite sensing range, limited field of view (FoV), or constrained egocentric observations; it does not by itself imply a formal safety guarantee under such sensing limits. Abbreviations: Dec. = decentralized execution; \(\triangle\) = partial or setting-dependent. The symbol $\dagger$ indicates that the safety guarantee is provided outside the decentralized multi-robot exploration setting.}
\label{tab:related_exploration}
\footnotesize
\setlength{\tabcolsep}{3.0pt}
\renewcommand{\arraystretch}{1.12}
\resizebox{\textwidth}{!}{%
\begin{tabular}{lcccccc}
\toprule
\textbf{Method} &
\textbf{Setting} &
\textbf{Dec.} &
\textbf{Limited Sensing} &
\textbf{Safety Guarantee} &
\textbf{Main Idea} &
\textbf{Gap to Ours} \\
\midrule
Frontier-based exploration~\cite{yamauchi_frontierbased_1997}
& Single & \(\triangle\) & \xmark & \xmark
& Frontier goals
& Known-map safety only \\

NBV / informative path planning~\cite{bircher_receding_2016,schmid_efficient_2020}
& Single & N/A & \(\triangle\) & \xmark
& Information gain
& Known-map safety only \\

Hierarchical frontier / NBV planners~\cite{dai_fast_2020,zhou_fuel_2021,cao_tare_2021,lindqvist_treebased_2024,liu_flare_2025}
& Single & N/A & \(\triangle\) & \xmark
& Frontier / NBV hierarchy
& No FoV guarantee \\

CoScan~\cite{dong_multirobot_2019}
& Multi & \xmark & \(\triangle\) & \xmark
& OMT-based task assignment
& No hidden-obstacle guarantee \\

Decentralized active information acquisition~\cite{atanasov_decentralized_2015,tzes_graph_2023}
& Multi & \checkmark & \(\triangle\) & \xmark
& Distributed information
& Estimation-focused \\

RACER~\cite{zhou_racer_2023}
& Multi & \checkmark & \(\triangle\) & \xmark
& Space decomposition
& No hidden-obstacle guarantee \\

Graph-based multi-robot exploration~\cite{varadharajan_multirobot_2025}
& Multi & \checkmark & \(\triangle\) & \xmark
& Local-global graphs
& Planner-level safety \\

Heterogeneous multi-robot exploration~\cite{zheng_aage_2025,yang_heha_2025,li_decentralized_2026}
& Hetero. Multi & \(\triangle\) & \(\triangle\) & \xmark
& Capability-aware planning
& Empirical safety \\

NeuralCoMapping~\cite{ye_multirobot_2022}
& Multi & \xmark & \checkmark & \xmark
& Graph matching
& No safety layer \\

Neural mapping/navigation~\cite{chaplot_learning_2020,ramakrishnan_poni_2022,li_learningaugmented_2023,cao_deep_2024}
& Single & N/A & \checkmark & \xmark
& Learned exploration
& Empirical safety \\

Neural multi-agent exploration~\cite{yu_learning_2022,yang_active_2024}
& Multi & \(\triangle\) & \checkmark & \xmark
& Neural maps / graphs
& Empirical safety \\

Decentralized RL exploration~\cite{calzolari_reinforcement_2025,chiun_marvel_2025}
& Multi & \checkmark & \checkmark & \xmark
& Learned policies
& Empirical safety \\

Perception-aware safe navigation~\cite{bena_safetyaware_2023,kim_visibilityaware_2025,firoozi_oampc_2025}
& Single & N/A & \checkmark & \checkmark$^\dagger$
& Visibility / occlusion safety
& Not multi-robot exploration \\

\textbf{Ours: \texttt{SEAMLiS}}
& \textbf{Multi} & \textbf{\checkmark} & \textbf{\checkmark} & \textbf{\checkmark}
& \textbf{\texttt{gatekeeper} + CBF}
& \textbf{Execution-layer safety} \\
\bottomrule
\end{tabular}%
}
\end{table*}

\subsection{Frontier, Next-Best-View, and Hierarchical Exploration}

Frontier-based exploration remains a standard abstraction for autonomous mapping, where the robot repeatedly selects goals on the boundary between known-free and unknown space~\cite{yamauchi_frontierbased_1997}. Subsequent work has improved exploration efficiency by refining frontier scoring, next-best-view selection, and trajectory generation. Receding-horizon next-best-view methods sample candidate viewpoints or branches and select paths with high expected information gain~\cite{bircher_receding_2016}, while sampling-based informative path planning optimizes longer-horizon information gain and can maintain a global candidate tree online~\cite{schmid_efficient_2020}. Recent exploration systems further improve scalability through incremental frontier structures, hierarchical local--global planning, next-best-trajectory optimization, and unknown-region-guided trajectory generation~\cite{dai_fast_2020,zhou_fuel_2021,cao_tare_2021,lindqvist_treebased_2024,liu_flare_2025}. These methods are effective for reducing exploration time, but safety is typically enforced only with respect to the currently known map or by a downstream local planner. Consequently, they do not by themselves address the case in which finite range, limited FoV, and optimistic planning allow a robot to move toward hidden obstacles that are detected too late.

\subsection{Optimization-Based and Decentralized Multi-Robot Exploration}

Multi-robot exploration can reduce mapping time by distributing sensing and planning across the team. CoScan~\cite{dong_multirobot_2019} coordinates multiple robots for dense scene reconstruction by extracting task views, assigning them through an Optimal Mass Transport (OMT)-based objective, and computing per-robot scanning paths; however, the original system relies on centralized reconstruction and planning. Decentralized active information acquisition instead formulates multi-robot sensing as distributed information gathering under graph-structured communication~\cite{atanasov_decentralized_2015,tzes_graph_2023}. For aerial robots, RACER~\cite{zhou_racer_2023} introduced a fully decentralized multi-UAV exploration system that uses online space decomposition, limited asynchronous communication, workload balancing, and hierarchical planning.

Recent distributed multi-robot exploration systems extend this direction to large-scale and unstructured environments.
For example, the multi-robot exploration planner in~\cite{varadharajan_multirobot_2025} maintains local and global exploration graphs that are synchronized among robots within communication range, enabling distributed exploration in subterranean and Mars-like environments.
Heterogeneous exploration systems further exploit complementary platform capabilities.
AAGE uses an aerial robot to provide a bird's-eye-view prior for ground-robot exploration~\cite{zheng_aage_2025}, while recent heterogeneous planners explicitly account for different traversal, sensing, and observation capabilities during task allocation and path planning~\cite{yang_heha_2025,li_decentralized_2026}.
These systems improve coverage efficiency and adaptability in realistic environments, but their collision avoidance remains primarily planner-level or empirical rather than a formal finite-FoV hidden-obstacle safety guarantee.

Learning-assisted assignment has also been used for multi-robot active mapping; NeuralCoMapping formulates robot-frontier assignment as a neural bipartite graph matching problem to improve efficiency and map completeness~\cite{ye_multirobot_2022}. These methods address coordination and exploration efficiency, but they generally do not provide formal collision-avoidance guarantees for initially unknown obstacles under bounded actuation and limited FoV. \texttt{SEAMLiS} is therefore complementary: it can sit below such goal-allocation and planning modules as an execution-layer safety mechanism.

\subsection{Learning-Based Exploration}

Learning-based exploration methods exploit learned spatial representations, priors, or policies to improve efficiency in visually rich environments. Active Neural SLAM uses a modular hierarchy with learned SLAM, learned global and local policies, and analytical planning to avoid the sample complexity of purely end-to-end exploration~\cite{chaplot_learning_2020}. Related single-agent neural mapping and navigation methods learn object-goal potentials, model-based spatial priors, or large-scale exploration policies~\cite{ramakrishnan_poni_2022,li_learningaugmented_2023,cao_deep_2024}. Multi-agent extensions learn cooperative policies from visual or map observations. MAANS extends Active Neural SLAM to cooperative visual exploration through a transformer-based multi-agent spatial planner~\cite{yu_learning_2022}, while Active Neural Topological Mapping for Multi-Agent Exploration constructs compact topological graphs and uses a hierarchical graph neural planner for goal selection~\cite{yang_active_2024}. More recent learning-based multi-robot exploration methods further study decentralized communication, density-based frontier selection, constrained-FoV viewpoint decisions, multi-task pretraining, and hybrid learning-planning architectures~\cite{calzolari_reinforcement_2025,chiun_marvel_2025,zhu_taskexp_2025,liu_vorlexplore_2026}. Although these approaches can improve exploration efficiency, scalability, and generalization, their safety properties are usually empirical and depend on the simulator, training distribution, or a separate collision checker. They do not certify recursive feasibility of collision avoidance when hidden obstacles become visible only after the remaining braking distance is insufficient.

\subsection{Safety-Aware Navigation under Limited Perception}

A separate line of work studies safety under limited sensing, occlusion, and visibility constraints. Safety-aware perception methods use CBF-based reasoning to choose sensor orientations or viewpoints that reduce collision risk~\cite{bena_safetyaware_2023}. Visibility-aware planning combines sampling-based planning with barrier-style constraints to account for limited sensing in unknown environments~\cite{kim_visibilityaware_2025}. Occlusion-aware MPC reasons about unseen dynamic obstacles through reachable sets and incorporates occlusion-induced risk into real-time planning~\cite{firoozi_oampc_2025}. These methods show that perception constraints are not merely sensing details but can directly determine whether a safety controller remains feasible. However, they are primarily developed for single-robot navigation or occlusion-aware planning rather than decentralized multi-robot exploration with arbitrary upstream goal assignment. In contrast, \texttt{SEAMLiS} targets the execution layer of multi-robot exploration: the attitude filter maintains visibility of the critical known-free/unknown boundary, while the positional CBF/HOCBF filter handles known obstacles, newly detected hidden obstacles, and inter-agent separation.

\section{PRELIMINARIES}

\subsection{Notation}

Let $\RealSpace$ denote the set of real numbers. For a vector $\va \in \RealSpace^n$, $\|\va\|$ denotes the Euclidean norm. For angles, $\operatorname{wrap}(\cdot)$ denotes wrapping onto $(-\pi,\pi]$. For a set $\calA \subset \RealSpace^d$, $\partial \calA$ denotes its boundary and $\overline{\calA}$ denotes its closure, i.e., $\overline{\calA} \coloneqq \calA \cup \partial \calA$.

\subsection{Robot Dynamics with Decoupled Yaw}

We consider $R$ robots indexed by $i \in \{1,\ldots,R\}$. Each robot is modeled by a double integrator in position and a single integrator in yaw, which is a standard simplified model for quadrotor translational motion with decoupled heading control:
\begin{equation}
\dot{\vp}_i = \vv_i,\qquad 
\dot{\vv}_i = \va_i,\qquad
\dot{\theta}_i = \omega_i,
\label{eq:dyn}
\end{equation}
where $\vp_i \in \RealSpace^2$ is position, $\vv_i \in \RealSpace^2$ is velocity, $\theta_i \in \RealSpace$ is yaw, and the control inputs are acceleration $\va_i \in \RealSpace^2$ and yaw-rate $\omega_i \in \RealSpace$. We assume bounded actuation:
\begin{equation}
\|\vv_i\| \le v_{\max},\qquad \|\va_i\| \le a_{\max},\qquad |\omega_i| \le \omega_{\max}.
\label{eq:bounds}
\end{equation}
The worst-case braking distance is defined as
\begin{equation}
d_{\textup{br}} \coloneqq \frac{v_{\max}^2}{2 a_{\max}},
\label{eq:braking_distance}
\end{equation}
which upper-bounds the distance required to stop from any admissible speed under maximum deceleration.

\subsection{Limited Sensing Model}

Each robot carries a perception sensor with finite sensing range $l_{\textup{range}}$ and limited horizontal FoV $\theta_{\textup{fov}} \in (0,2\pi]$. At time $t$, robot $i$ senses the footprint
\begin{equation}\label{eq:fov_footprint}
\begin{aligned}
 \calF_{i}(t) \coloneqq 
\Big\{\vp \in \calW \mid &\ \|\vp - \vp_i(t)\| \le l_{\textup{range}}, \\
& \left|\operatorname{wrap}\big(\angle(\vp - \vp_i(t)) - \theta_i(t)\big)\right|
\le \frac{\theta_{\textup{fov}}}{2}, \\
& \operatorname{LoS}(\vp_i(t),\vp)=1
\Big\},
\end{aligned}
\end{equation}
where $\calW \subset \RealSpace^2$ denotes the bounded workspace and $\angle(\cdot)$ returns the polar angle of a nonzero vector in $\RealSpace^2$. The line-of-sight indicator $\operatorname{LoS}(\vp,\vq)$ equals $1$ if the line segment connecting $\vp$ and $\vq$ is not occluded by an obstacle, and equals $0$ otherwise.

\section{PROBLEM FORMULATION}

\subsection{Environment, Map Representation, and Frontiers}

Let the true static obstacle set be $\calO^\star \subset \calW$, and let the true free space be $\calS \coloneqq \calW \setminus \calO^\star$. We decompose $\calO^\star$ into the initially known obstacle set $\calO_0$ and the initially unknown, or hidden, obstacle set $\calH$:
\begin{equation}
\calO^\star = \calO_0 \cup \calH,\qquad \calO_0 \cap \calH = \emptyset.
\end{equation}
Each robot $i$ maintains a local map, such as an occupancy grid or Euclidean signed distance map, constructed from its onboard sensor stream.

\begin{assumption}[State estimation and local perception]\label{ass:slam}
Each robot has access to its own pose and velocity estimates, and the estimation and local mapping errors are bounded by the conservative margins used in the safety constraints. Within the sensed footprint $\calF_i(t)$, occupied and free regions are classified correctly up to these margins.
\end{assumption}

Such conservative margins can be supplied by certifiably-correct mapping pipelines that deflate reported free space under odometry drift~\cite{agrawal_certifiablycorrect_2025}.

Under \autoref{ass:slam}, the local map of robot $i$ partitions $\calW$ into a known-free region $\calB_i(t)$, a known-occupied region $\calO_i(t)$, and an unknown region $\calU_i(t)$. We model the known-free region as the accumulated sensed footprint excluding occupied cells:
\begin{equation}
\calB_i(t) \coloneqq \left(\bigcup_{\tau\in[t_0,t]} \calF_i(\tau)\right) \setminus \calO_i(t),
\label{eq:known_free}
\end{equation}
and the known obstacle set as
\begin{equation}
\calO_i(t) \coloneqq \calO_0 \cup \hat{\calO}_i(t),
\end{equation}
where $\hat{\calO}_i(t)$ denotes the subset of hidden obstacles in $\calH$ that have been detected by robot $i$ up to time $t$. The unknown region is then
\begin{equation}
\calU_i(t) \coloneqq \calW \setminus \left(\overline{\calB_i(t)} \cup \calO_i(t)\right).
\end{equation}

\begin{definition}[Frontier~\cite{yamauchi_frontierbased_1997}]\label{def:frontier}
A \emph{frontier point} for robot $i$ at time $t$ is a boundary point of the known-free region that is adjacent to the unknown region. The frontier set is defined as
\begin{equation}
\calF^{\textup{fr}}_i(t)
\coloneqq
\partial \calB_i(t) \cap \overline{\calU_i(t)} .
\label{eq:frontier}
\end{equation}
\end{definition}

\subsection{Decentralized Multi-Robot Exploration Stack}

We adopt the following decentralized exploration setting.

\begin{assumption}[Decentralized coordination]\label{ass:decentralized}
The robot team operates without a central controller. Each robot $i$ determines its own control input $\vu_i=(\va_i,\omega_i)$ using its local state and local map, together with limited information shared asynchronously by neighboring robots, such as poses, velocities estimated from pose histories, and selected frontier candidates.
\end{assumption}

The decentralized exploration stack consists of the following modules.

\textbf{(i) Goal assignment.} A possibly asynchronous multi-robot goal allocation module assigns each robot a goal $\vp^{\textup{goal}}_i(t)\in \calW$, typically chosen on or near a frontier:
\begin{equation}
\vp^{\textup{goal}}_i(t)\in \calF^{\textup{fr}}(t),
\end{equation}
where $\calF^{\textup{fr}}(t)$ is a set of frontier candidates collected from local and communicated information. We treat the goal allocator as a black box; examples include frontier-based exploration~\cite{yamauchi_frontierbased_1997}, decentralized adaptations of CoScan-style task-assignment objectives~\cite{dong_multirobot_2019}, or learning-based exploration policies.

\textbf{(ii) Local path planning.} Given $\vp^{\textup{goal}}_i(t)$ and the local map, a planner computes a waypoint path $\calP_i(t)$. To maintain progress under incomplete maps, especially when goals originate from other robots, it is common to plan optimistically by treating unknown cells as traversable~\cite{gao_online_2016,zhong_informationdriven_2022} or by assigning them a finite cost penalty~\cite{bena_safetyaware_2023}. In our implementation, we use A* for simplicity, but \texttt{SEAMLiS} does not depend on a specific planner.

\textbf{(iii) Local control.} A positional controller tracks $\calP_i(t)$, while an attitude controller outputs $\omega_i$ to regulate the sensor heading and improve sensing efficiency.

\subsection{Safety and Exploration Objective}

The exploration objective is to map the workspace by driving the union of known-free regions to cover the true free space:
\begin{equation}
\exists\, T_f < \infty:\quad
\bigcup_{i=1}^{R} \calB_i(T_f) = \calS.
\label{eq:coverage}
\end{equation}
Simultaneously, we require safety for all time:
\begin{subequations}\label{eq:safety_goal}
\begin{align}
\forall t\ge t_0,\ \forall i: \, &
\operatorname{Robot}(\vp_i(t)) \cap \calO^\star = \emptyset, \\
\forall t\ge t_0,\ \forall i\neq \ell: \, &
\operatorname{Robot}(\vp_i(t)) \cap \operatorname{Robot}(\vp_\ell(t)) = \emptyset,
\end{align}
\end{subequations}
where $\operatorname{Robot}(\vp)$ is modeled as a disk of radius $r_{\textup{rob}}$ centered at $\vp$. We aim to reduce the total exploration time while satisfying~\eqref{eq:safety_goal}. Since direct time-optimal exploration is difficult to encode at the controller level, we use a local visibility-promoting attitude objective that increases information gain per unit motion.

The key difficulty is that~\eqref{eq:safety_goal} cannot be certified solely from $\calB_i(t)$ and $\calO_i(t)$ under limited sensing: unknown obstacles in $\calH$ may lie in $\calU_i(t)$ and become detectable only when they enter $\calF_i(t)$. Therefore, safety depends on \emph{timely detection}: obstacles must be sensed early enough that the positional safety controller remains feasible under the actuation bounds~\eqref{eq:bounds}. This motivates two execution-layer safety mechanisms. First, a positional safety filter must guarantee avoidance of obstacles and robots that are already represented in the local information state. Second, an attitude safety filter must ensure that the boundary between known-free and unknown space is observed early enough to make future obstacle avoidance feasible. We next introduce the two safety-filter mechanisms underlying \texttt{SEAMLiS} and then describe how they are instantiated in the proposed framework.

\section{SAFETY-FILTER MECHANISMS}
\label{sec:safety_filter_mechanisms}

This section summarizes the two safety-filter mechanisms used by \texttt{SEAMLiS}. The first is a CBF/HOCBF-based filter, which minimally modifies a nominal control input so that the closed-loop system remains inside a controlled-invariant set defined by the already known safety constraints. The second is a \texttt{gatekeeper} algorithm, which decides whether a nominal policy can be executed for a certified duration before switching to a backup policy. These tools are introduced here at an abstract level; their \texttt{SEAMLiS}-specific instantiations are given in \autoref{sec:seamlis_method}.

\subsection{CBF-Based Safety Filtering}
\label{subsec:cbf_building_block}

Consider a continuous-time control-affine system
\begin{equation}
\dot{\vx}=f(\vx)+g(\vx)\vu,
\label{eq:sf_control_affine}
\end{equation}
where $\vx\in\StateSpace\subset\Rn$ is the state, $\vu\in\ControlSpace\subset\Rm$ is the control input, and $\ControlSpace$ denotes the admissible input set. The vector fields $f:\StateSpace\to\Rn$ and $g:\StateSpace\to\Rnm$ are assumed to be locally Lipschitz. Let the desired safe set be the zero-superlevel set
\begin{equation}
\calC \coloneqq \{\vx\in\StateSpace \mid h(\vx)\ge 0\},
\label{eq:sf_safe_set}
\end{equation}
where $h:\StateSpace\to\RealSpace$ is continuously differentiable.

\begin{definition}[\ac{CBF}~\cite{ames_control_2019}]
\label{def:sf_cbf}
The function $h$ is a \ac{CBF} for~\eqref{eq:sf_control_affine} on $\calC$ if there exists an extended class-$\calK_{\infty}$ function $\alpha$ such that
\begin{equation}
\sup_{\vu\in\ControlSpace}
\left[
\lieder_f h(\vx)+\lieder_g h(\vx)\vu
\right]
\ge -\alpha(h(\vx)),
\qquad \forall \vx\in\calC,
\label{eq:sf_cbf_condition}
\end{equation}
where $\lieder_f h$ and $\lieder_g h$ denote the Lie derivatives of $h$ along $f$ and $g$, respectively.
\end{definition}

The corresponding set of CBF-admissible inputs is
\begin{equation}
K_{\textup{cbf}}(\vx)
\coloneqq
\left\{
\vu\in\ControlSpace
\ \middle|\
\lieder_f h(\vx)+\lieder_g h(\vx)\vu
\ge -\alpha(h(\vx))
\right\}.
\label{eq:sf_kcbf} \nonumber
\end{equation}

\begin{theorem}[\cite{ames_control_2019}]
\label{thm:sf_cbf_forward_invariance}
Suppose that $h$ is a \ac{CBF} for $\calC$ and that the feedback controller $\vu(\vx)$ is locally Lipschitz on $\StateSpace$ and satisfies $\vu(\vx)\in K_{\textup{cbf}}(\vx)$ for all $\vx\in\calC$. If $\vx(t_0)\in\calC$, then $\calC$ is forward invariant for the closed-loop system~\eqref{eq:sf_control_affine}; equivalently, $\vx(t)\in\calC$ for all $t\ge t_0$.
\end{theorem}

For safety functions with relative degree greater than one, the control input does not appear in the first derivative of $h$. We therefore use high-order CBFs (HOCBFs)~\cite{xiao_control_2019}. Let $h$ have relative degree $r$ with respect to~\eqref{eq:sf_control_affine}. Define
\begin{subequations}
\label{eq:sf_hocbf_functions}
\begin{align}
\psi_0(\vx) &\coloneqq h(\vx), \nonumber \\
\psi_q(\vx) &\coloneqq \dot{\psi}_{q-1}(\vx)+\alpha_q(\psi_{q-1}(\vx)),
\quad q=1,\ldots,r-1, \nonumber \\
\psi_r(\vx,\vu) &\coloneqq
\lieder_f\psi_{r-1}(\vx)+\lieder_g\psi_{r-1}(\vx)\vu
+\alpha_r(\psi_{r-1}(\vx)), \nonumber
\end{align}
\end{subequations}
where $\alpha_1,\ldots,\alpha_r$ are class-$\calK$ functions. For $q=1,\ldots,r$, define the intermediate safe sets
\begin{equation}
\calC_q
\coloneqq
\{\vx\in\StateSpace \mid \psi_{q-1}(\vx)\ge 0\}.
\label{eq:sf_hocbf_sets}
\end{equation}

\begin{definition}[HOCBF~\cite{xiao_control_2019}]
\label{def:sf_hocbf}
The function $h$ is a HOCBF of relative degree $r$ for~\eqref{eq:sf_control_affine} if there exist class-$\calK$ functions $\alpha_1,\ldots,\alpha_r$ such that
\begin{equation}
\sup_{\vu\in\ControlSpace}\psi_r(\vx,\vu)\ge 0,
\qquad
\forall \vx\in \bigcap_{q=1}^{r}\calC_q .
\label{eq:sf_hocbf_condition}
\end{equation}
The associated set of HOCBF-admissible inputs is
\begin{equation}
K_{\textup{hocbf}}(\vx)
\coloneqq
\left\{
\vu\in\ControlSpace
\ \middle|\
\psi_r(\vx,\vu)\ge 0
\right\}.
\label{eq:sf_khocbf}
\end{equation}
\end{definition}

\begin{theorem}[\cite{xiao_control_2019}]
\label{thm:sf_hocbf_forward_invariance}
Suppose that $h$ is a HOCBF of relative degree $r$ and that the feedback controller $\vu(\vx)$ is locally Lipschitz and satisfies $\vu(\vx)\in K_{\textup{hocbf}}(\vx)$ for all $\vx\in\cap_{q=1}^{r}\calC_q$. If $\vx(t_0)\in \bigcap_{q=1}^{r}\calC_q$,
then $\cap_{q=1}^{r}\calC_q$ is forward invariant for the closed-loop system~\eqref{eq:sf_control_affine}.
\end{theorem}

In \autoref{subsec:seamlis_positional_cbf}, the distance-based collision-avoidance constraints have relative degree two with respect to the acceleration input. For such constraints, we use the standard second-order HOCBF form with linear class-$\calK$ functions $\alpha_0,\alpha_1>0$:
\begin{equation}
\psi(\vx,\vu)
\coloneqq
\ddot{h}(\vx,\vu)+\alpha_1\dot{h}(\vx)+\alpha_0 h(\vx)\ge 0.
\label{eq:sf_second_order_hocbf}
\end{equation}
This yields the obstacle and inter-agent constraints $\psi_{ij}^{\textup{obs}}$ and $\psi_{i\ell}^{\textup{ag}}$ used by the positional safety filter.

\subsection{\texttt{gatekeeper} Algorithm}
\label{subsec:gatekeeper_building_block}

We next introduce a recent backup-based safety filter: \texttt{gatekeeper}~\cite{agrawal_gatekeeper_2023, agrawal_gatekeeper_2024}. Let $\calC\subseteq\StateSpace$ be a safe set and let $\calS_0\subseteq\calC$ be a terminal controlled-invariant set. Let $\pi_{\textup{nom}}:\StateSpace\to\ControlSpace$ denote a task-oriented nominal policy, and let $\pi_{\textup{b}}:\StateSpace\to\ControlSpace$ denote a backup policy.

For any locally Lipschitz feedback policy $\pi:\StateSpace\to\ControlSpace$, let $\varphi_t^{\pi}(\vx_0)$ denote the state at time $t\ge0$ of the closed-loop system $\dot{\vx}=f(\vx)+g(\vx)\pi(\vx)$ initialized at $\vx(0)=\vx_0$.

\begin{definition}[Recoverable set~\cite{kim_backupbased_2026}]
\label{def:gk_recoverable_set}
Let $\calS_0\subseteq\calC\subseteq\StateSpace$, let $\pi:\StateSpace\to\ControlSpace$ be a feedback policy, and let $T>0$. The recoverable set of $\calS_0$ under $\pi$ over horizon $T$ is
\begin{equation}
\begin{aligned}
\calR^{\pi}(T;\calS_0)
\coloneqq
\big\{
\vx_0\in\calC
\ \big|\
&\varphi_t^{\pi}(\vx_0)\in\calC,\ \forall t\in[0,T],\\
&\varphi_T^{\pi}(\vx_0)\in\calS_0
\big\}.
\end{aligned}
\label{eq:gk_recoverable_set}
\end{equation}
\end{definition}

\begin{definition}[Backup policy~\cite{kim_backupbased_2026}]
\label{def:gk_backup_policy}
Let $\calS_0\subseteq\calC\subseteq\StateSpace$ and let $T_B>0$. A feedback policy $\pi_{\textup{b}}:\StateSpace\to\ControlSpace$ is a backup policy for $\calS_0$ if $\calS_0$ is forward invariant under $\pi_{\textup{b}}$ and the recoverable set $\calR^{\pi_{\textup{b}}}(T_B;\calS_0)$ contains a neighborhood of $\calS_0$.
\end{definition}

Thus, the backup policy keeps the system inside $\calS_0$ once it reaches $\calS_0$, and it can drive nearby states into $\calS_0$ within the recovery horizon $T_B$ without leaving $\calC$.

Let $\Delta t>0$ be the control-update interval of the digital safety monitor, and let $t_k\coloneqq t_0+k\Delta t$, $k\in\mathbb{Z}_{\ge0}$. At each update time, the monitor evaluates whether the nominal policy can be executed for a switching time $T_S\ge0$ before handing control to the backup policy.

\begin{definition}[Candidate trajectory~\cite{agrawal_gatekeeper_2024,kim_backupbased_2026}]
\label{def:gk_candidate_trajectory}
Given the current state $\vx_k=\vx(t_k)$ and a switching time $T_S\ge0$, the candidate trajectory is
\begin{equation}
\begin{split}
\raisetag{7.5ex}
\chi^{\textup{can}}(\tau;\vx_k,T_S)
\coloneqq
\begin{cases}
\varphi_{\tau}^{\pi_{\textup{nom}}}(\vx_k),
& \tau\in[0,T_S],\\[1mm]
\varphi_{\tau-T_S}^{\pi_{\textup{b}}}
\!\left(
\varphi_{T_S}^{\pi_{\textup{nom}}}(\vx_k)
\right),
& \tau\in[T_S,T_S+T_B].
\end{cases}
\label{eq:gk_candidate_trajectory}
\end{split}
\end{equation}
\end{definition}

The candidate trajectory first follows the nominal policy for duration $T_S$, and then follows the backup policy for duration $T_B$.

\begin{definition}[Validity indicator]
\label{def:gk_validity}
For fixed $T_B>0$, the candidate trajectory is valid if it remains in $\calC$ over the full candidate horizon and reaches $\calS_0$ at the terminal time:
{\footnotesize
\begin{equation}
\begin{split}
\raisetag{5.5ex}
\operatorname{Valid}(\vx_k;T_S,T_B)
\coloneqq
\begin{cases}
1, &
\begin{array}{l}
\chi^{\textup{can}}(\tau;\vx_k,T_S)\in\calC,\quad
\forall \tau\in[0,T_S+T_B],\\
\chi^{\textup{can}}(T_S+T_B;\vx_k,T_S)\in\calS_0,
\end{array}\\[3mm]
0, & \text{otherwise}.
\end{cases}
\label{eq:gk_validity}
\end{split}
\end{equation}
}
When $T_B$ is fixed, we write $\operatorname{Valid}(\vx_k;T_S)$ for notational simplicity.
\end{definition}

Let $T_H>0$ be the search horizon for the nominal segment. The \texttt{gatekeeper} algorithm computes
\begin{equation}
T_S^{\star}(\vx_k)
\coloneqq
\sup
\left\{
T_S\in[0,T_H]
\ \middle|\
\operatorname{Valid}(\vx_k;T_S)=1
\right\}.
\label{eq:gk_ts_star_general}
\end{equation}
The applied policy is then
\begin{equation}
\pi_{\textup{GK}}(\vx_k)
=
\begin{cases}
\pi_{\textup{nom}}(\vx_k), & T_S^{\star}(\vx_k)>0,\\
\pi_{\textup{b}}(\vx_k), & T_S^{\star}(\vx_k)=0.
\end{cases}
\label{eq:gk_decision_general}
\end{equation}
In implementation, the search in~\eqref{eq:gk_ts_star_general} is performed over a finite grid $\mathcal{T}_H=\{0,\Delta t,2\Delta t,\ldots,T_H\}$. In \autoref{subsec:seamlis_attitude_gatekeeper}, we construct the attitude safety filter using a visibility-promoting yaw policy as the nominal policy and a velocity-tracking yaw policy as the backup policy.

\section{\texttt{SEAML\MakeLowercase{i}S}: SAFETY FILTERS FOR DECENTRALIZED EXPLORATION}
\label{sec:seamlis_method}

\subsection{Overview}

\texttt{SEAMLiS} integrates the two safety-filter mechanisms introduced in \autoref{sec:safety_filter_mechanisms} into the execution layer of a decentralized exploration stack. The upstream goal assignment and local path-planning modules are treated as interchangeable components. As shown in Fig.~\ref{fig:seamlis_arch}, robot $i$ receives a waypoint path $\calP_i(t)$ from the local planner, maintains the local map components $\calB_i(t)$ and $\calO_i(t)$ from onboard sensing, and applies two safety filters:
\begin{enumerate}[label=\roman*)]
    \item a \textbf{\texttt{gatekeeper}-based attitude filter}, which maintains sufficient visibility of the critical boundary between known-free and unknown space; and
    \item a \textbf{\ac{CBF}-based positional filter}, which avoids known obstacles, newly detected obstacles, and neighboring robots.
\end{enumerate}

The attitude filter acts on the yaw-rate input $\omega_i$ and decides whether the robot can safely execute a nominal visibility-promoting yaw policy or must switch to a conservative velocity-tracking backup policy. The positional filter acts on the acceleration input $\va_i$ and enforces high-order CBF constraints for collision avoidance. Figure~\ref{fig:seamlis_concept} illustrates the perception-limited failure mode addressed by the proposed execution-layer safety filters and the corresponding visibility-aware correction.

\begin{figure}[t]
    \centering
    \includegraphics[width=0.99\linewidth]{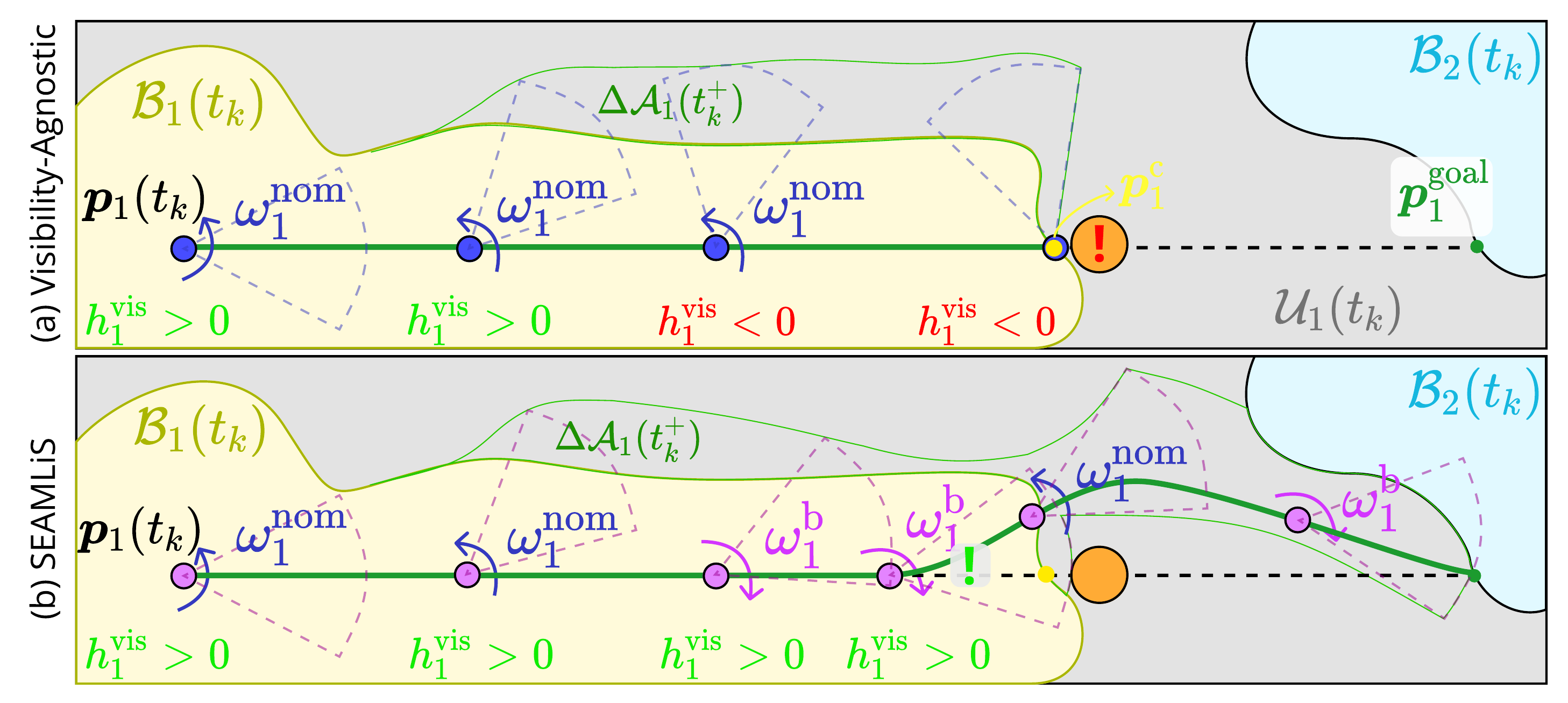}
    \caption{Illustration of the perception-limited safety issue and the proposed \texttt{SEAMLiS} algorithm. The visibility-agnostic execution in (a) follows a nominal visibility-promoting yaw command $\omega_1^{\textup{nom}}$ for robot~1 that maximizes the incremental sensed area $\Delta\calA_1$, but the critical point $\vp_1^{\textup{c}}$ on the boundary between the robot~1's local known-free region $\calB_1(t_k)$ and the unknown region $\calU_1(t_k)$ can remain outside the FoV, i.e., $h_1^{\textup{vis}}(\cdot)<0$. In this case, an obstacle near the critical point may be detected after the available stopping distance is insufficient. In (b), \texttt{SEAMLiS} uses the \texttt{gatekeeper} filter to switch from the nominal yaw command to the backup velocity-tracking yaw command $\omega_1^{\textup{b}}$ whenever needed, keeping $h_1^{\textup{vis}}(\cdot) \ge 0$ while still allowing nominal visibility maximization when it is certified safe.}
    \label{fig:seamlis_concept}
\end{figure}

\subsection{Nominal Local Policies}
\label{subsec:seamlis_nominal_policies}

\subsubsection{Nominal positional tracking}

Let $\va_i^{\textup{nom}}(t)$ denote the nominal acceleration command generated by an arbitrary local path-tracking controller for the waypoint path $\calP_i(t)$. The proposed safety layer does not depend on the specific tracking controller; it only requires the nominal command to be filtered through the admissible acceleration set defined in \autoref{subsec:seamlis_positional_cbf}.

\subsubsection{Nominal visibility-promoting attitude policy}

The nominal yaw policy $\pi_{\theta}^{\textup{vis}}$ greedily maximizes the incremental sensed area of previously unobserved space. At each control step, robot $i$ enumerates a finite set of candidate headings
$\Theta \coloneqq \{\theta^{(k)}\}_{k=1}^{K}$. For each candidate heading, it estimates
\begin{equation}
\Delta \calA_i^{(k)}(t)
\coloneqq
\operatorname{Area}
\Big(
\calF_i(t;\theta^{(k)})
\setminus
\big(\calB_i(t)\cup\calO_i(t)\big)
\Big),
\label{eq:delta_area}
\end{equation}
and selects
\begin{equation}
\theta_i^{\star}(t)
\in
\argmax_{\theta^{(k)}\in\Theta}
\Delta \calA_i^{(k)}(t).
\label{eq:nominal_heading_selection}
\end{equation}
The resulting nominal yaw-rate command is
\begin{equation}
\omega_i^{\textup{nom}}(t)
=
\operatorname{sat}_{\omega_{\max}}
\!\left(
k_{\theta}\operatorname{wrap}
\big(\theta_i^{\star}(t)-\theta_i(t)\big)
\right),
\label{eq:yaw_nominal}
\end{equation}
where $k_{\theta}>0$ and $\operatorname{sat}_{\omega_{\max}}(\cdot)$ saturates to $[-\omega_{\max},\omega_{\max}]$.

The nominal policy improves exploration efficiency by steering the finite-FoV sensor toward information-rich regions. However, it may rotate the sensor away from the direction of motion. Under finite sensing range and limited FoV, this can delay the detection of hidden obstacles along the upcoming trajectory. Therefore, \texttt{SEAMLiS} does not directly apply~\eqref{eq:yaw_nominal}; instead, it uses the attitude safety filter in \autoref{subsec:seamlis_attitude_gatekeeper}.

\subsection{Visibility-Safety Set for the Attitude Filter}
\label{subsec:visibility_safe_set}

For analysis, we represent the planner output $\calP_i(t)$ as a continuous curve. Fix time $t$, and let $\gamma_i:[0,L_i(t)]\to\calW$ denote an arc-length parameterization of the current planned path of length $L_i(t)$, for example the piecewise-linear interpolation of the waypoint sequence $\calP_i(t)$, with $\gamma_i(0)=\vp_i(t)$.

\begin{definition}[Critical point]
\label{def:critical_point}
The critical point $\vp_i^{\textup{c}}(t)$ is the first intersection of the planned path with the boundary of the known-free region:
\begin{equation}
\vp_i^{\textup{c}}(t)
\coloneqq
\gamma_i\!\left(s_i^{\textup{c}}(t)\right),
\quad
s_i^{\textup{c}}(t)
\coloneqq
\inf
\left\{
s\ge0
\ \middle|\
\gamma_i(s)\in\partial\calB_i(t)
\right\}.
\label{eq:critical_def}
\end{equation}
If the planned path does not intersect $\partial\calB_i(t)$ over the look-ahead horizon, we set $s_i^{\textup{c}}(t)=+\infty$.
\end{definition}

The critical point identifies the location beyond which the robot begins to rely on the optimistic unknown-is-free planning assumption. To preserve safety under bounded deceleration, the robot must observe this boundary before the remaining path distance becomes smaller than the required braking distance.

Let
\begin{equation}
d_i^{\textup{c}}(t)\coloneqq s_i^{\textup{c}}(t)
\label{eq:critical_distance}
\end{equation}
denote the path distance to the critical point. When $d_i^{\textup{c}}(t)<+\infty$, define the critical-point direction angle
\begin{equation}
\beta_i^{\textup{c}}(t)
\coloneqq
\angle\!\left(\vp_i^{\textup{c}}(t)-\vp_i(t)\right).
\label{eq:critical_bearing}
\end{equation}
The angular violation outside the FoV is
\begin{equation}
\Delta\theta_i^{\textup{c}}(t)
\coloneqq
\max
\left\{
0,\
\left|
\operatorname{wrap}
\big(
\beta_i^{\textup{c}}(t)-\theta_i(t)
\big)
\right|
-
\frac{\theta_{\textup{fov}}}{2}
\right\}.
\label{eq:delta_theta}
\end{equation}
Under the yaw-rate bound $|\omega_i|\le\omega_{\max}$, the fastest admissible rotation that brings the critical point into the FoV requires $\Delta\theta_i^{\textup{c}}(t)/\omega_{\max}$ seconds. This recovery time is achieved when the yaw controller rotates the sensor toward the direction of the critical point at the maximum admissible rate. During this recovery interval, the robot can translate by at most $v_{\max}\Delta\theta_i^{\textup{c}}(t)/\omega_{\max}$.

\begin{definition}[Visibility constraint function]
\label{def:visibility_constraint}
For $d_i^{\textup{c}}(t)<+\infty$, define
\begin{equation}
h_i^{\textup{vis}}(t)
\coloneqq
d_i^{\textup{c}}(t)
-
d_{\textup{br}}
-
\frac{v_{\max}}{\omega_{\max}}
\Delta\theta_i^{\textup{c}}(t).
\label{eq:hvis}
\end{equation}
The state is \emph{visibility-safe} if $h_i^{\textup{vis}}(t)\ge0$. If $d_i^{\textup{c}}(t)=+\infty$, the visibility constraint is inactive and we set $h_i^{\textup{vis}}(t)=+\infty$.
\end{definition}

The condition $h_i^{\textup{vis}}(t)\ge0$ means that the robot has enough remaining path distance to rotate the critical point into the FoV while still retaining one worst-case braking distance. The braking-distance term $d_{\textup{br}}$ is analogous to the recoverability margin in backup-based safety filters~\cite{kim_backupbased_2026}: after a hidden obstacle is detected, the positional safety filter must still have enough distance to stop or execute an avoidance maneuver. A less conservative implementation can replace $d_{\textup{br}}$ by the instantaneous braking distance $\|\vv_i(t)\|^2/(2a_{\max})$.

For the attitude safety filter, it is useful to view the visibility margin
in~\eqref{eq:hvis} as a function of the robot state, parameterized by the
local path and the known-free set used by the monitor. Let $\vx_i \coloneqq [\vp_i^\top,\vv_i^\top,\theta_i]^\top $. We write $h_i^{\textup{vis}}(\vx_i;\calP_i(t),\calB_i(t))$ for the value of~\eqref{eq:hvis} obtained by computing the critical point
along $\calP_i(t)$ with respect to $\partial\calB_i(t)$ and evaluating the
resulting distance and angular terms at state $\vx_i$. With the same
convention, $d_i^{\textup{c}}(\vx_i;\calP_i(t),\calB_i(t))$ and
$\Delta\theta_i^{\textup{c}}(\vx_i;\calP_i(t),\calB_i(t))$ denote the
corresponding critical-point distance and angular violation outside the FoV.

The visibility-safe set for robot $i$ is then defined as
\begin{equation}
\calC_i^{\textup{vis}}(t)
\coloneqq
\left\{
\vx_i\in \StateSpace \mid
h_i^{\textup{vis}}(\vx_i;\calP_i(t),\calB_i(t))\ge0
\right\}.
\label{eq:visibility_safe_set}
\end{equation}
We also define the terminal visibility set
\begin{equation}
\begin{aligned}
\calS_{0,i}^{\textup{vis}}(t)
\coloneqq
\Big\{
\vx_i\in\calC_i^{\textup{vis}}(t)
\ \Big|\
& d_i^{\textup{c}}(\vx_i;\calP_i(t),\calB_i(t))=+\infty \\
& \text{or }
\Delta\theta_i^{\textup{c}}(\vx_i;\calP_i(t),\calB_i(t))=0
\Big\}.
\end{aligned} \nonumber
\label{eq:visibility_terminal_set}
\end{equation}
Thus, $\calS_{0,i}^{\textup{vis}}(t)$ consists of visibility-safe states
in which either no known-free/unknown boundary appears along the current
look-ahead path, or the critical point is already inside the FoV.

\subsection{\texttt{gatekeeper}-Based Attitude Safety Filter}
\label{subsec:seamlis_attitude_gatekeeper}

For robot $i$, the safe set, terminal set, and policies in
\autoref{subsec:gatekeeper_building_block} are chosen as
\[
\calC = \calC_i^{\textup{vis}}(t_k),\quad
\calS_0 = \calS_{0,i}^{\textup{vis}}(t_k),\quad
\pi_{\textup{nom}} = \pi_{\theta}^{\textup{vis}},\quad
\pi_{\textup{b}} = \pi_{\theta}^{\textup{b}},
\]
where the sets are evaluated using the local path $\calP_i(t_k)$ and
known-free set $\calB_i(t_k)$ available at the current update time $t_k$.
The control input is the yaw rate $\omega_i$. The nominal policy
$\pi_{\theta}^{\textup{vis}}$ is the visibility-promoting yaw policy
defined in~\eqref{eq:yaw_nominal}, and the backup policy
$\pi_{\theta}^{\textup{b}}$ is the velocity-tracking yaw policy defined
below. The attitude filter therefore preserves the visibility constraint
$h_i^{\textup{vis}}\ge0$ while allowing the nominal visibility-promoting
yaw command whenever its future execution is certified safe.

\subsubsection{Backup velocity-tracking attitude policy}

The backup yaw policy aligns the sensor with the instantaneous direction of
translation. Define
\begin{equation}
\theta_i^{\textup{vel}}(t)
\coloneqq
\begin{cases}
\angle(\vv_i(t)), & \|\vv_i(t)\|>v_{\min},\\
\theta_i(t), & \text{otherwise},
\end{cases}
\label{eq:velocity_heading}
\end{equation}
where $v_{\min}>0$ is a small velocity threshold. The backup yaw-rate
command is
\begin{equation}
\omega_i^{\textup{b}}(t)
\coloneqq
\operatorname{sat}_{\omega_{\max}}
\!\left(
k_{\textup{b}}
\operatorname{wrap}
\big(
\theta_i^{\textup{vel}}(t)-\theta_i(t)
\big)
\right),
\label{eq:yaw_backup}
\end{equation}
with $k_{\textup{b}}>0$. This policy prioritizes motion-aligned sensing near
unknown space and serves as the backup policy $\pi_{\theta}^{\textup{b}}$.

For the safety analysis, $k_{\textup{b}}$ is chosen large enough that the
backup command saturates whenever the velocity direction lies outside the
FoV:
\begin{equation}
k_{\textup{b}}
\ge
\frac{2\omega_{\max}}{\theta_{\textup{fov}}}.
\label{eq:backup_gain_condition}
\end{equation}
Thus, whenever the direction of motion is outside the FoV, the backup policy
rotates the sensor toward that direction at the maximum admissible yaw rate
$\omega_{\max}$. This is the yaw-rate bound used in the visibility margin
in~\eqref{eq:hvis}.

We also choose the backup horizon to cover the worst-case angular recovery
time:
\begin{equation}
T_B
\ge
\frac{
\max\left\{0,\pi-\frac{\theta_{\textup{fov}}}{2}\right\}
}{
\omega_{\max}
}.
\label{eq:backup_horizon_condition}
\end{equation}
With these choices, the immediate-backup candidate $T_S=0$ is valid from
any visibility-safe update state. Indeed, if
$\vx_i(t_k)\in\calC_i^{\textup{vis}}(t_k)$, then
$h_i^{\textup{vis}}(\vx_i(t_k);\calP_i(t_k),\calB_i(t_k))\ge0$ already
reserves both the worst-case distance traveled while rotating the critical
point into the FoV and the braking distance $d_{\textup{br}}$. Therefore,
executing the backup yaw policy immediately keeps the candidate rollout
inside $\calC_i^{\textup{vis}}(t_k)$ and reaches
$\calS_{0,i}^{\textup{vis}}(t_k)$ within the backup horizon.

\subsubsection{Candidate yaw rollout and validity}

Let $t_k$ be the current update time. During the attitude-filter prediction,
the local path and known-free set are fixed at their current values,
$\calP_i(t_k)$ and $\calB_i(t_k)$. For a switching time
$T_S\in[0,T_H]$ and a backup horizon $T_B>0$, let
\[
\hat{\vx}_i(\tau;T_S)
\coloneqq
[
\hat{\vp}_i(\tau)^\top,
\hat{\vv}_i(\tau)^\top,
\hat{\theta}_i(\tau;T_S)
]^\top
\]
denote the predicted state used by the attitude monitor at prediction time
$\tau\in[0,T_S+T_B]$. The translational component
$(\hat{\vp}_i,\hat{\vv}_i)$ is propagated conservatively along the current
planned path $\calP_i(t_k)$ with speed bounded by $v_{\max}$. The predicted
yaw evolves according to
\begin{equation}
\dot{\hat{\theta}}_i(\tau;T_S)
=
\omega_i^{\textup{can}}(\tau;T_S),
\qquad
\hat{\theta}_i(0;T_S)=\theta_i(t_k),
\label{eq:predicted_yaw_dynamics}
\end{equation}
where the candidate yaw-rate policy is
{\footnotesize
\begin{equation}
\omega_i^{\textup{can}}(\tau;T_S)
=
\begin{cases}
\pi_{\theta}^{\textup{vis}}
\big(
\hat{\vx}_i(\tau;T_S);
\calP_i(t_k),\calB_i(t_k),\calO_i(t_k)
\big),
& \tau\in[0,T_S),\\[1mm]
\pi_{\theta}^{\textup{b}}
\big(
\hat{\vx}_i(\tau;T_S)
\big),
& \hspace{-3.0em}\tau\in[T_S,T_S+T_B].
\end{cases}
\label{eq:candidate_yaw}
\end{equation}
}
Here $\pi_{\theta}^{\textup{vis}}$ is the nominal visibility-promoting yaw
policy in~\eqref{eq:yaw_nominal}, and $\pi_{\theta}^{\textup{b}}$ is the
backup velocity-tracking yaw policy in~\eqref{eq:yaw_backup}.

Using the state-based visibility margin defined in
\autoref{subsec:visibility_safe_set}, the predicted visibility margin is
\begin{equation}
\hat h_i^{\textup{vis}}(\tau;T_S)
\coloneqq
h_i^{\textup{vis}}
\big(
\hat{\vx}_i(\tau;T_S);
\calP_i(t_k),\calB_i(t_k)
\big).
\label{eq:predicted_hvis}
\end{equation}
The \texttt{SEAMLiS}-specific validity indicator is then
{\footnotesize
\begin{equation}
\operatorname{Valid}_i^{\textup{vis}}(t_k;T_S)
\coloneqq
\begin{cases}
1, &
\begin{array}{l}
\hat h_i^{\textup{vis}}(\tau;T_S)\ge0,\quad
\forall \tau\in[0,T_S+T_B],\\
\hat{\vx}_i(T_S+T_B;T_S)\in\calS_{0,i}^{\textup{vis}}(t_k),
\end{array}\\[3mm]
0, & \text{otherwise}.
\end{cases}
\label{eq:seamlis_validity}
\end{equation}
}
Equivalently, a candidate yaw rollout is valid if it remains in the
visibility-safe set $\calC_i^{\textup{vis}}(t_k)$ over the full candidate
horizon and reaches the terminal visibility set
$\calS_{0,i}^{\textup{vis}}(t_k)$ after the backup segment.

The filter searches over the discrete set
\[
\mathcal{T}_H
=
\{0,\Delta t,2\Delta t,\ldots,T_H\}
\]
and computes
\begin{equation}
T_{S,i}^{\star}(t_k)
\coloneqq
\max
\left\{
T_S\in\mathcal{T}_H
\ \middle|\
\operatorname{Valid}_i^{\textup{vis}}(t_k;T_S)=1
\right\}.
\label{eq:seamlis_ts_star}
\end{equation}

Here, $T_{S,i}^{\star}(t_k)$ is the robot-$i$ specialization of the general \texttt{gatekeeper} switching time $T_S^{\star}(\vx_k)$ in~\eqref{eq:gk_ts_star_general}, evaluated at the current state $\vx_i(t_k)$ with the local planning and mapping quantities $\calP_i(t_k)$, $\calB_i(t_k)$, and $\calO_i(t_k)$ held fixed during the candidate-rollout evaluation.

The candidate $T_S=0$ corresponds to immediately applying the backup yaw
policy. By the backup-policy construction above, this immediate-backup
candidate is valid whenever the current state is visibility-safe and the
backup horizon satisfies~\eqref{eq:backup_horizon_condition}. Therefore,
the search set in~\eqref{eq:seamlis_ts_star} is nonempty for visibility-safe
update states. Thus, $T_S=0$ is a certified backup candidate, not an
artificial fallback.

The applied yaw-rate command is
\begin{equation}
\omega_i(t_k)
=
\begin{cases}
\omega_i^{\textup{nom}}(t_k), & T_{S,i}^{\star}(t_k)>0,\\
\omega_i^{\textup{b}}(t_k), & T_{S,i}^{\star}(t_k)=0.
\end{cases}
\label{eq:seamlis_yaw_policy}
\end{equation}

Thus, the nominal visibility-promoting yaw policy is used whenever there
exists a positive certified duration before the backup policy must be
invoked. If no positive nominal duration is certified, the filter immediately
applies the velocity-tracking backup yaw.

\begin{algorithm}[t]
    \footnotesize
    \SetAlgoLined
    \DontPrintSemicolon
    \caption{\texttt{gatekeeper} Attitude Filter for Robot $i$}
    \label{alg:seamlis_gatekeeper}
    \KwIn{Current state $(\vp_i,\vv_i,\theta_i)$; planned path $\calP_i$; known-free set $\calB_i$; known obstacle set $\calO_i$; horizons $T_H,T_B$}
    \KwOut{Committed yaw-rate $\omega_i$}
    \BlankLine
    Compute $\vp_i^{\textup{c}}$ and $d_i^{\textup{c}}$ using \autoref{def:critical_point}\;
    \If{$d_i^{\textup{c}}=+\infty$}{
        \Return $\omega_i^{\textup{nom}}$\;
    }
    Compute $\Delta\theta_i^{\textup{c}}$ using~\eqref{eq:delta_theta} and $h_i^{\textup{vis}}$ using~\eqref{eq:hvis}\;
    $T_{S,i}^{\star}\gets 0$\;
    \For{$T_S=T_H,T_H-\Delta t,\ldots,0$}{
        Construct $\omega_i^{\textup{can}}(\cdot;T_S)$ using~\eqref{eq:candidate_yaw}\;
        Predict $\hat h_i^{\textup{vis}}(\cdot;T_S)$ and evaluate~\eqref{eq:seamlis_validity}\;
        \If{$\operatorname{Valid}_i^{\textup{vis}}(t_k;T_S)=1$}{
            $T_{S,i}^{\star}\gets T_S$\;
            \textbf{break}\;
        }
    }
    \If{$T_{S,i}^{\star}>0$}{
        \Return $\omega_i^{\textup{nom}}$ \tcp*{Nominal visibility-promoting yaw}
    }
    \Else{
        \Return $\omega_i^{\textup{b}}$ \tcp*{Backup velocity-tracking yaw}
    }
\end{algorithm}

\subsection{CBF-Based Positional Safety Filter}
\label{subsec:seamlis_positional_cbf}

We construct the positional safety filter by applying the CBF formulation in \autoref{subsec:cbf_building_block} for the positional subsystem
\[
\dot{\vp}_i=\vv_i,\qquad
\dot{\vv}_i=\va_i,
\]
with positional state $\vx_i^{\textup{p}}\coloneqq[\vp_i^\top,\vv_i^\top]^\top$ and control input $\va_i$. The nominal input is the path-tracking command $\va_i^{\textup{nom}}(t)$, and the CBF filter modifies this command only as needed to satisfy obstacle-avoidance and inter-agent separation constraints.

\paragraph{Obstacle avoidance constraints}

For each obstacle $j$ known to robot $i$ at time $t$, including newly detected obstacles, define the inflated safety radius
\[
r_{\textup{safe}}
\coloneqq
r_{\textup{rob}}+r_{\textup{obs}}+\epsilon,
\]
where $r_{\textup{obs}}$ is the obstacle-radius approximation and $\epsilon>0$ accounts for mapping and tracking uncertainty. Let $\vp_j^{\textup{obs}}$ be the obstacle center. Define
\begin{equation}
h_{ij}^{\textup{obs}}(\vx_i^{\textup{p}})
\coloneqq
\|\vp_i-\vp_j^{\textup{obs}}\|^2
-
r_{\textup{safe}}^2.
\label{eq:hcol}
\end{equation}
Under~\eqref{eq:dyn}, $h_{ij}^{\textup{obs}}$ has relative degree two with respect to $\va_i$. Applying the second-order HOCBF form in~\eqref{eq:sf_second_order_hocbf} gives
\begin{equation}
\begin{aligned}
\psi_{ij}^{\textup{obs}}(\vx_i^{\textup{p}},\va_i)
& \coloneqq\; 2(\vp_i-\vp_j^{\textup{obs}})^\top \va_i
+2\|\vv_i\|^2\\
&+2\alpha_1(\vp_i-\vp_j^{\textup{obs}})^\top\vv_i
+\alpha_0 h_{ij}^{\textup{obs}}(\vx_i^{\textup{p}})
\ge0,
\end{aligned}
\label{eq:hocbf_obs}
\end{equation}
where $\alpha_0,\alpha_1>0$.

\paragraph{Inter-agent collision avoidance constraints}

For each neighboring robot $\ell\neq i$ in the communication neighborhood $\calN_i(t)$, define
\[
r_{\textup{ag}}
\coloneqq
2r_{\textup{rob}}+\epsilon_{\textup{ag}},
\]
where $\epsilon_{\textup{ag}}>0$ accounts for state-estimation and communication uncertainty. Define
\begin{equation}
h_{i\ell}^{\textup{ag}}(\vx_i^{\textup{p}},\vx_\ell^{\textup{p}})
\coloneqq
\|\vp_i-\vp_\ell\|^2
-
r_{\textup{ag}}^2.
\label{eq:hag}
\end{equation}
Let $\Delta\vp_{i\ell}\coloneqq\vp_i-\vp_\ell$ and $\Delta\vv_{i\ell}\coloneqq\vv_i-\vv_\ell$. We use a decentralized pairwise HOCBF allocation~\cite{lee_distributed_2025}, obtained from the same second-order HOCBF condition with the responsibility split across the two agents:
\begin{equation}
\begin{aligned}
\psi_{i\ell}^{\textup{ag}}(\vx_i^{\textup{p}},\vx_\ell^{\textup{p}},\va_i)
&\coloneqq\; 2\Delta\vp_{i\ell}^{\top}\va_i
+\|\Delta\vv_{i\ell}\|^2\\
&+\alpha_1\Delta\vp_{i\ell}^{\top}\Delta\vv_{i\ell}
+\frac{\alpha_0}{2}
h_{i\ell}^{\textup{ag}}(\vx_i^{\textup{p}},\vx_\ell^{\textup{p}})
\ge0.
\end{aligned}
\label{eq:hocbf_agent}
\end{equation}
When both agents enforce their allocated constraints, the sum of the two inequalities recovers the standard relative-degree-two pairwise HOCBF condition for inter-agent separation.

\paragraph{Admissible acceleration set}

Let $\calJ_i(t)$ denote the index set of obstacles currently known to robot $i$. The admissible acceleration set is
\begin{equation}
\begin{aligned}
K_i^{\textup{cbf}}(t)
\coloneqq
\Big\{
\va_i\in\RealSpace^2
\ \Big|\
&\|\va_i\|\le a_{\max},\\
&\psi_{ij}^{\textup{obs}}(\vx_i^{\textup{p}},\va_i)\ge0,\quad
\forall j\in\calJ_i(t),\\
&\psi_{i\ell}^{\textup{ag}}(\vx_i^{\textup{p}},\vx_\ell^{\textup{p}},\va_i)\ge0,\quad
\forall \ell\in\calN_i(t)
\Big\}.
\end{aligned} \nonumber
\label{eq:Kcbf} 
\end{equation}
A CBF Quadratic Program (QP) implementation is
\begin{equation}
\va_i(t)
\in
\argmin_{\va\in K_i^{\textup{cbf}}(t)}
\|\va-\va_i^{\textup{nom}}(t)\|^2.
\label{eq:cbf_qp}
\end{equation}
Alternatively, the continuous-time HOCBF constraints can be converted into discrete-time constraints under a zero-order-hold assumption on the control input and then imposed within an MPC-CBF formulation~\cite{kim_how_2025}.

\subsection{Safety Guarantees}
\label{subsec:seamlis_safety_guarantees}

We now state sufficient conditions under which the combined attitude and
positional filters guarantee collision avoidance under limited sensing. The
attitude filter does not require the existence of a positive certified
nominal execution time. If no $T_S>0$ is certified, the immediate-backup
candidate $T_S=0$ is valid by construction, and the filter applies the
velocity-tracking backup yaw policy.

\begin{assumption}[Sensing range]
\label{ass:range}
The sensing range satisfies
\begin{equation}
l_{\textup{range}}
\ge
r_{\textup{safe}}+d_{\textup{br}}.
\label{eq:sensing_range_condition}
\end{equation}
\end{assumption}

\begin{lemma}[Feasibility from timely detection]
\label{lem:stop}
Consider robot $i$ under~\eqref{eq:dyn}--\eqref{eq:bounds}. If an obstacle is detected at time $t$ with clearance
\[
\|\vp_i(t)-\vp^{\textup{obs}}\|
\ge
r_{\textup{safe}}+d_{\textup{br}},
\]
then there exists an admissible acceleration input $\va_i(\cdot)$ that prevents collision with that obstacle for all future time.
\end{lemma}

\begin{proof}
Apply maximum braking,
\[
\va_i(\tau)
=
-a_{\max}\frac{\vv_i(\tau)}{\|\vv_i(\tau)\|}
\]
whenever $\|\vv_i(\tau)\|>0$, and set $\va_i=\mathbf{0}$ once the robot stops. This input satisfies $\|\va_i\|\le a_{\max}$ and stops the robot within distance at most $\|\vv_i(t)\|^2/(2a_{\max})\le d_{\textup{br}}$. Therefore, an initial clearance of at least $r_{\textup{safe}}+d_{\textup{br}}$ ensures that the braking trajectory cannot enter the inflated obstacle disk of radius $r_{\textup{safe}}$.
\end{proof}

\begin{lemma}[Critical-point visibility implies timely detection]
\label{lem:critical_detection}
Suppose \autoref{ass:range} holds and robot $i$ satisfies $h_i^{\textup{vis}}(t)\ge0$ for all $t$. Then any hidden obstacle in $\calH$ that intersects the robot's immediate future swept path beyond the known-free boundary is detected before the robot reaches a state from which collision avoidance would require more than the worst-case braking distance.
\end{lemma}

\begin{proof}
By \autoref{def:critical_point}, $\vp_i^{\textup{c}}(t)$ is the first point along the planned path where the robot exits the known-free region. The condition $h_i^{\textup{vis}}(t)\ge0$ guarantees that the robot has enough path distance to rotate this boundary point into the FoV while still retaining distance $d_{\textup{br}}$. Since $l_{\textup{range}}\ge r_{\textup{safe}}+d_{\textup{br}}$, any obstacle intersecting the future swept path at or immediately beyond this boundary is sensed with at least the clearance required by \autoref{lem:stop}.
\end{proof}

\begin{theorem}[Collision avoidance under limited sensing]
\label{thm:safety}
Under Assumptions~\ref{ass:slam} and~\ref{ass:range}, suppose each robot is
initialized in the visibility-safe set, i.e.,
$\vx_i(t_0)\in\calC_i^{\textup{vis}}(t_0)$, and executes the attitude filter
in \autoref{alg:seamlis_gatekeeper} with the backup yaw policy satisfying
\eqref{eq:backup_gain_condition} and the backup horizon satisfying
\eqref{eq:backup_horizon_condition}. Suppose also that the positional CBF
filter in \autoref{subsec:seamlis_positional_cbf} is applied, that
$\calJ_i(t)$ is updated whenever robot $i$ detects a previously hidden
obstacle, and that $K_i^{\textup{cbf}}(t)$ is nonempty for all $t\ge t_0$.
Then the robots avoid collisions with initially known obstacles, previously
unknown obstacles, and other robots for all time.
\end{theorem}

\begin{proof}
For initially known obstacles and neighboring robots, the positional controller selects $\va_i(t)\in K_i^{\textup{cbf}}(t)$. Therefore, the obstacle and inter-agent safe sets defined by~\eqref{eq:hcol} and~\eqref{eq:hag} are forward invariant by the HOCBF condition in \autoref{thm:sf_hocbf_forward_invariance}.

For previously unknown obstacles, the attitude filter maintains
$h_i^{\textup{vis}}(t)\ge0$. At any visibility-safe update state, the
candidate $T_S=0$ is valid by the backup-policy construction above. Hence,
the \texttt{gatekeeper} search set is nonempty. If a positive switching time
is certified, the filter applies the nominal yaw command for the current
control interval; otherwise, it applies the backup yaw command. In both
cases, the selected yaw command is the first segment of a valid candidate
rollout, and therefore the predicted state remains in
$\calC_i^{\textup{vis}}$. Since the robots are initialized in
$\calC_i^{\textup{vis}}$, induction over the attitude-filter update times
gives $h_i^{\textup{vis}}(t)\ge0$ for all $t\ge t_0$.

By \autoref{lem:critical_detection}, any hidden obstacle on the immediate
future swept path is detected with sufficient clearance. By
\autoref{lem:stop}, there then exists an admissible acceleration input that
avoids collision. Once the obstacle is detected, it is inserted into
$\calJ_i(t)$, and the corresponding CBF constraint~\eqref{eq:hocbf_obs} is
enforced by the positional safety filter. Repeating this argument at each
detection event proves collision avoidance for all time.
\end{proof}

\begin{remark}
The safety argument depends on local sensing, bounded dynamics, recursive feasibility of the attitude safety filter, and nonemptiness of the positional CBF-admissible set. It does not depend on the specific upstream goal assignment or path-planning algorithm. Therefore, \texttt{SEAMLiS} can be combined with frontier-based methods~\cite{yamauchi_frontierbased_1997}, decentralized adaptations of optimization-based task-assignment objectives such as CoScan~\cite{dong_multirobot_2019}, or learning-based exploration policies~\cite{chaplot_learning_2020,calzolari_reinforcement_2025}, including asynchronous decentralized systems that do not share full occupancy maps.
\end{remark}

\section{SIMULATION RESULTS}
\label{sec:results}

\subsection{Experimental Setup}
\label{subsec:sim_setup}

We evaluate \texttt{SEAMLiS} in simulation on decentralized exploration and mapping tasks under limited sensing, using the robot dynamics~\eqref{eq:dyn} and sensing model~\eqref{eq:fov_footprint}. We consider teams of $R\in\{1,2,3\}$ robots. Unless otherwise stated, the sensing and actuation limits are fixed to
$\theta_{\textup{fov}}=70^\circ$,
$l_{\textup{range}}=4.5$ m,
$v_{\max}=1.5$ m/s,
$a_{\max}=2.0$ m/s$^2$,
and $r_{\textup{rob}}=0.15$ m. All controllers and planners run in discrete time with sampling time $\Delta t=0.1$ s.

\paragraph{Randomized environments with known and unknown obstacles}
Each trial instantiates a random environment by sampling the initially known obstacle set $\calO_0$ and the initially unknown obstacle set $\calH$. The number of known obstacles is sampled uniformly from $[10,15]$, and the number of hidden obstacles is sampled uniformly from $[5,25]$. Obstacle locations and radii are sampled randomly. In addition to circular obstacles, we include rectangular obstacles represented by superellipsoids in $\calO_0$. All methods are evaluated on the same base workspace, occupancy-grid resolution, obstacle samples, and initial robot states within each trial.

\begin{figure*}[t]
    \centering
    \includegraphics[width=0.98\textwidth]{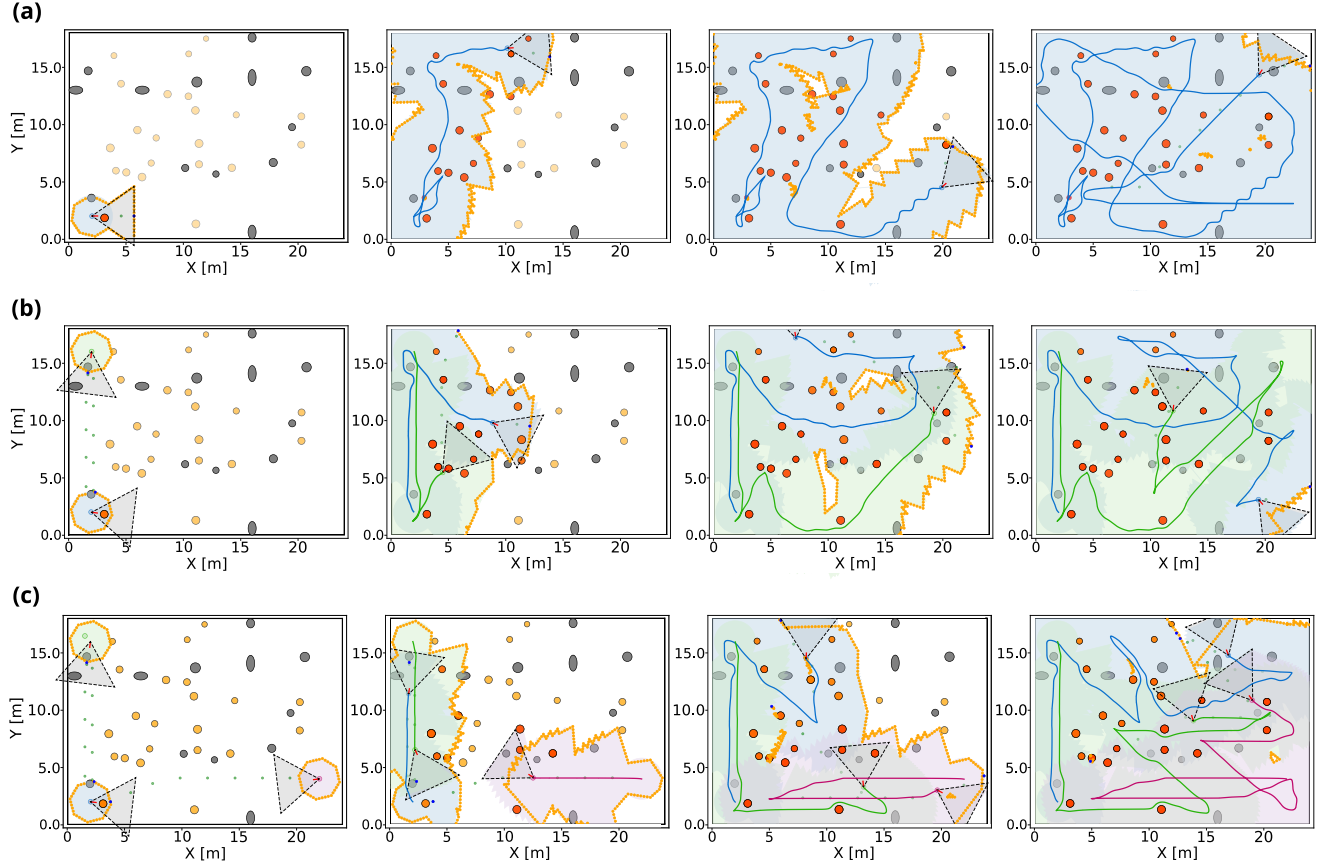}
    \caption{Representative time progressions of \texttt{SEAMLiS} for one-, two-, and three-robot exploration. Rows (a)-(c) correspond to $R=1$, $R=2$, and $R=3$, respectively, and the columns within each row show successive time instants. This example uses the decentralized CoScan-inspired allocator described in Sec.~\ref{subsec:sim_setup}. The gray obstacles denote the initially known obstacle set $\calO_0$. The light orange obstacles denote the initially unknown obstacle set $\calH$ that has not yet been detected, while the darker orange obstacles denote obstacles that have been detected by at least one robot. Although this visualization shows when an obstacle has been detected, the detected obstacle information is not shared across agents during execution, so each robot maintains its own local detected set. The colored curves show robot trajectories, the dashed sectors show finite-FoV sensing footprints $\calF_i(t)$, and the shaded regions show accumulated known-free regions $\calB_i(t)$. The yellow points denote frontier points, and the green points denote the set of A* waypoints along the current local plan. The snapshots illustrate that \texttt{SEAMLiS} continues to expand coverage while maintaining motion-relevant visibility near unknown regions.}
    \label{fig:exploration_progression}
\end{figure*}

\begin{figure}[t]
    \centering
    \includegraphics[width=0.98\linewidth]{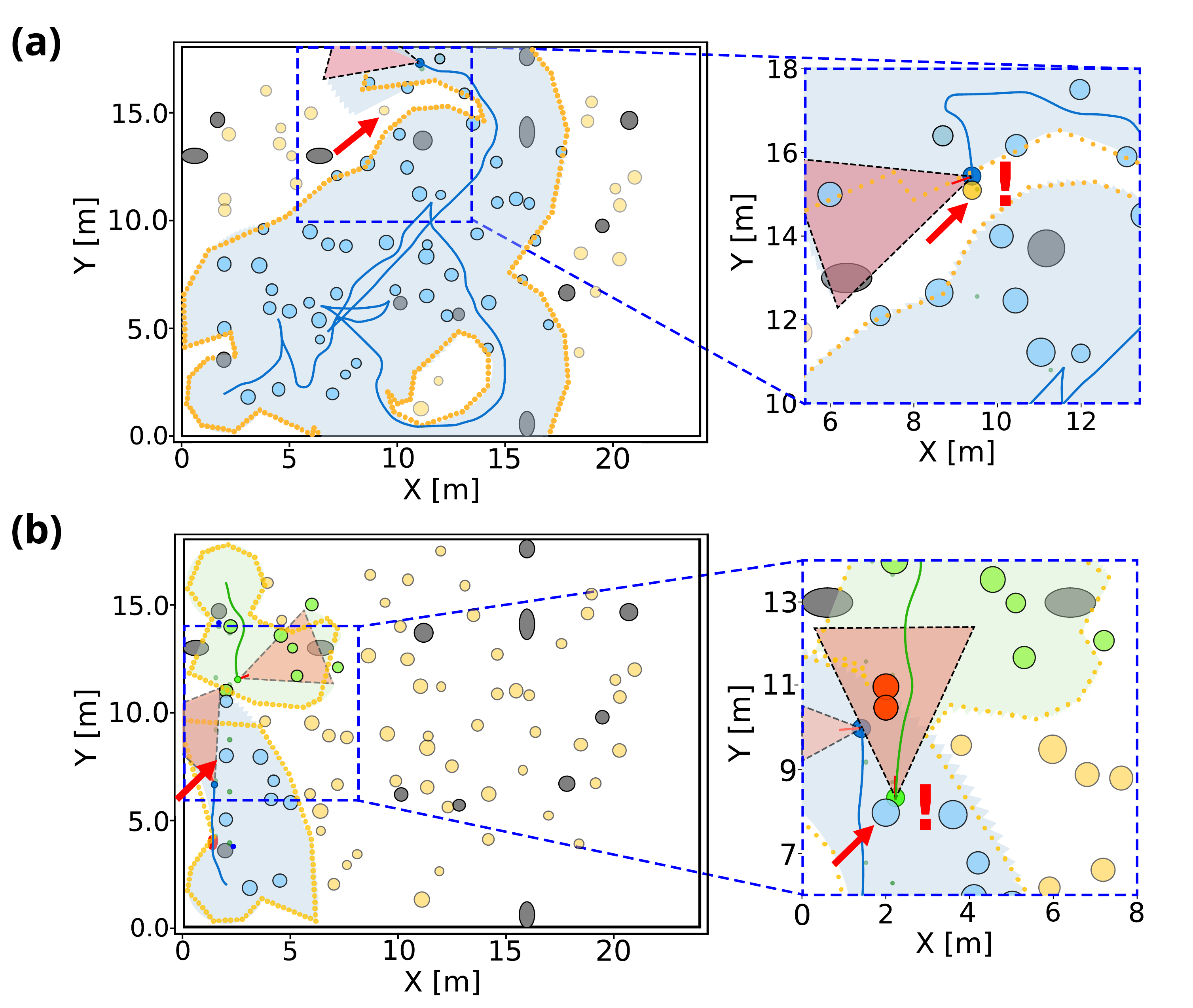}
    \caption{
    Representative failure cases of baseline attitude policies under limited
    sensing. These qualitative failure snapshots use $\theta_{\textup{fov}}=50^\circ$, while the quantitative benchmark in Table~\ref{tab:sim_benchmark} uses the default setting $\theta_{\textup{fov}}=70^\circ$. In each row, the left panel shows the exploration workspace shortly
    before collision, and the right panel shows a zoomed-in snapshot at collision.
    Gray obstacles are initially known obstacles. Light orange obstacles are
    hidden obstacles that have not yet been detected. When a hidden obstacle is
    detected, it is visualized using the color of the robot that detected it;
    dark orange indicates hidden obstacles detected by all robots. These colors
    are used only for visualization, and detected obstacle information is not
    shared as a full map during execution. 
    (a) Under the visibility-promoting (VP) policy with one robot, the robot
    prioritizes information gain and collides with a hidden obstacle that is not
    observed early enough to preserve a feasible avoidance maneuver.
    (b) Under constant yaw-rate in a two-robot setting, the green robot moves toward
    a region where the blue robot has more local information and collides with
    an obstacle detected by the blue robot but absent from the green robot's
    local map. These cases illustrate that, without full-map sharing,
    perception-limited failures become more pronounced in decentralized
    exploration.
    }
    \label{fig:baseline_failures}
\end{figure}

\paragraph{Fully decentralized and asynchronous execution}
A key aspect of the evaluation is that there is no centralized computation or decision making. Each robot runs the full exploration stack onboard and independently:
\begin{enumerate}[label=\roman*)]
    \item \textbf{Goal assignment.} We test two decentralized goal-assignment modules: (a) frontier-based exploration~\cite{yamauchi_frontierbased_1997} and (b) a decentralized CoScan-inspired allocator based on CoScan's OMT-style task-assignment objective~\cite{dong_multirobot_2019, ye_multirobot_2022}. The latter is not the original centralized CoScan algorithm. Instead, each robot performs goal assignment locally using its own map and asynchronously received frontier/task candidates. Nearby robots exchange only selected frontier/task information and poses; they do not exchange full occupancy grids, distance fields, obstacle lists, planned paths, or control inputs. Each robot replans on its own local timer, and one robot's local assignment computation does not overwrite or update the goals selected by neighboring robots.
    \item \textbf{Communication.} Robots share only their current poses and frontier information. They \textbf{\emph{do not}} exchange occupancy grids, distance fields, obstacle lists, planned paths, or control inputs. Messages are sent and received asynchronously, and each robot replans using its own local timer and the latest received messages.
    \item \textbf{Local path planning.} If an assigned frontier lies beyond the sensing range, the robot computes a waypoint path using A*~\cite{hart_formal_1968}. Collision-free planning is defined with respect to the robot's currently known obstacle set $\calO_i(t)$; unknown space $\calU_i(t)$ is treated optimistically as traversable.
    \item \textbf{Local control.} Each robot tracks the planned path using the positional safety controller together with one of the attitude-control configurations described below.
\end{enumerate}

\begin{table}[t]
\centering
\caption{Decentralized multi-robot exploration benchmark under limited sensing, averaged over 100 trials per setting. We report (a) Success: success rate, higher is better; (b) Collision: collision rate, lower is better; (c) Violation: mean number of unknown-region overlap violations per robot, lower is better; and (d) Time: exploration time over successful trials only, lower is better. D-CoScan denotes our decentralized CoScan-inspired allocator, which retains CoScan's OMT-style assignment objective but uses local asynchronous information sharing.}
\label{tab:sim_benchmark}
\scriptsize
\setlength{\tabcolsep}{4.2pt}
\renewcommand{\arraystretch}{1.15}
\begin{tabular}{@{} c c l c c c c @{}}
\toprule
 & $R$ & Att. Policy
& Success $\uparrow$
& Collision $\downarrow$
& Violation $\downarrow$
& Time [s] $\downarrow$ \\
\midrule
\multirow{12}{*}{\makecell[c]{Frontier\\-based\\exploration \\ \cite{yamauchi_frontierbased_1997}}}
& \multirow{4}{*}{1}
& Constant & 0.36 & \textbf{0.00} & 0.31 & 155.70 \\
& & VP & \textbf{1.00} & \textbf{0.00} & 2.27 & \textbf{108.63} \\
& & Vel.\ Track. & 0.00 & \textbf{0.00} & \textbf{0.00} & N/A \\
& & \texttt{SEAMLiS} & \textbf{1.00} & \textbf{0.00} & \textbf{0.00} & 140.15 \\
\cmidrule(lr){2-7}
& \multirow{4}{*}{2}
& Constant & 0.00 & 1.00 & 6.00 & N/A \\
& & VP & 0.53 & 0.47 & 2.83 & \textbf{58.50} \\
& & Vel.\ Track. & 0.00 & \textbf{0.00} & \textbf{0.00} & N/A \\
& & \texttt{SEAMLiS} & \textbf{1.00} & \textbf{0.00} & 1.97 & 59.55 \\
\cmidrule(lr){2-7}
& \multirow{4}{*}{3}
& Constant & 0.00 & 1.00 & 2.33 & N/A \\
& & VP & 0.00 & 1.00 & 2.33 & N/A \\
& & Vel.\ Track. & 0.00 & \textbf{0.00} & \textbf{0.00} & N/A \\
& & \texttt{SEAMLiS} & \textbf{1.00} & \textbf{0.00} & 4.33 & 48.50 \\
\midrule
\multirow{12}{*}{\makecell[c]{D-CoScan\\ inspired\\ by \cite{dong_multirobot_2019}}}
& \multirow{4}{*}{1}
& Constant & 0.69 & 0.31 & 64.76 & 143.51 \\
& & VP & \textbf{1.00} & \textbf{0.00} & 0.33 & 103.23 \\
& & Vel.\ Track. & 0.00 & \textbf{0.00} & \textbf{0.00} & N/A \\
& & \texttt{SEAMLiS} & \textbf{1.00} & \textbf{0.00} & 0.36 & \textbf{109.40} \\
\cmidrule(lr){2-7}
& \multirow{4}{*}{2}
& Constant & 0.00 & 0.37 & 54.10 & N/A \\
& & VP & 0.63 & 0.37 & 1.81 & 51.00 \\
& & Vel.\ Track. & 0.00 & \textbf{0.00} & \textbf{0.00} & N/A \\
& & \texttt{SEAMLiS} & \textbf{1.00} & \textbf{0.00} & 2.19 & \textbf{50.88} \\
\cmidrule(lr){2-7}
& \multirow{4}{*}{3}
& Constant & 0.00 & 1.00 & 16.49 & N/A \\
& & VP & 0.59 & 0.41 & 3.51 & \textbf{38.10} \\
& & Vel.\ Track. & 0.00 & \textbf{0.00} & \textbf{0.00} & N/A \\
& & \texttt{SEAMLiS} & \textbf{1.00} & \textbf{0.00} & 2.15 & 43.08 \\
\bottomrule
\end{tabular}
\end{table}

\paragraph{Controllers}
The positional controller is common to all methods. We use a \ac{MPC}-\ac{CBF} controller~\cite{zeng_safetycritical_2021} to track the current waypoint or goal while enforcing CBF constraints for known obstacles, newly detected obstacles, and inter-agent avoidance. The MPC cost is tuned to reach goals efficiently in free space, while the CBF parameters are tuned to avoid collisions smoothly once obstacles are detected with sufficient clearance.

\begin{figure*}[t]
    \centering
    \includegraphics[width=0.8\textwidth]{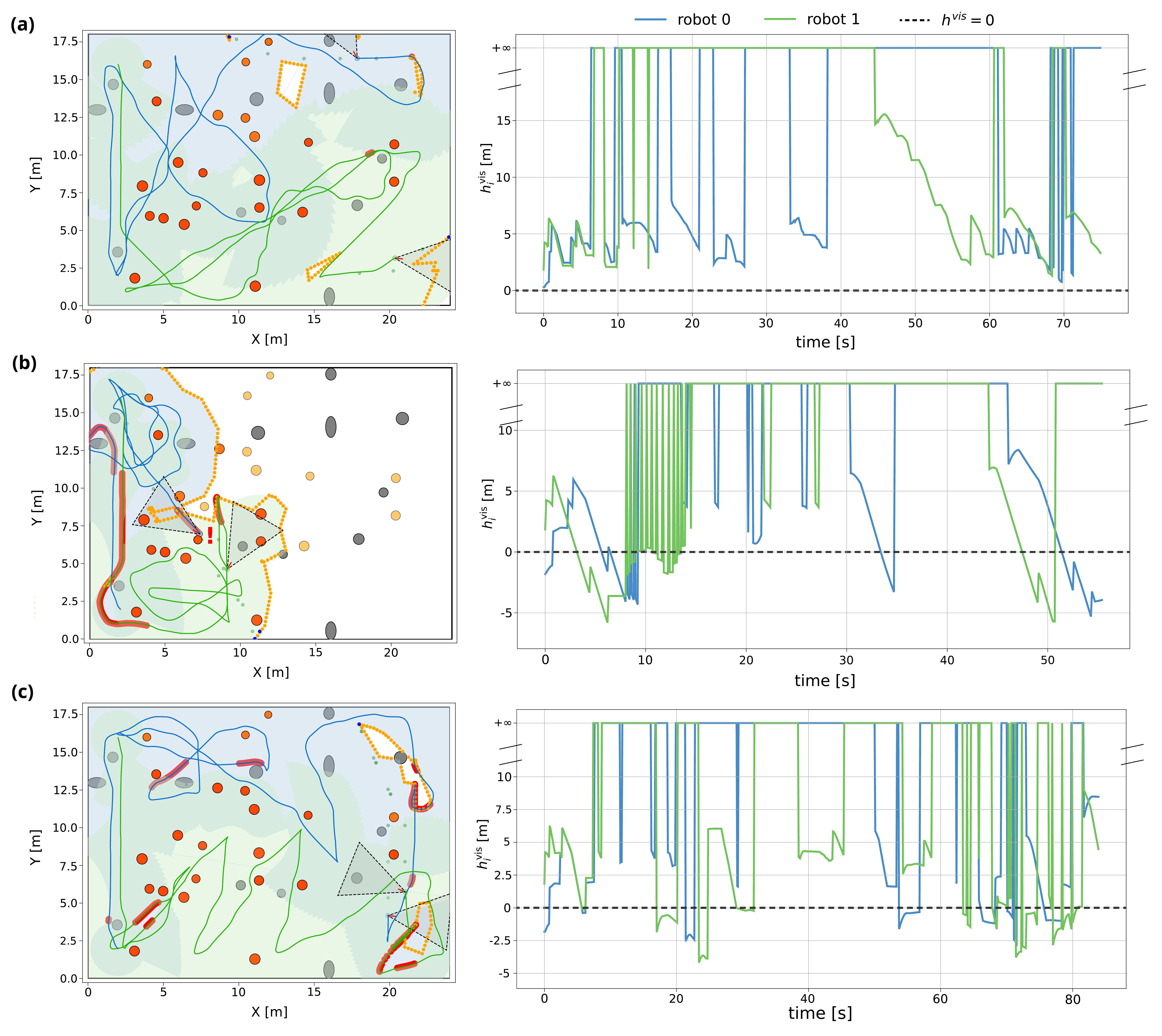}
    \caption{
    Visibility-margin evolution in a representative two-robot benchmark trial.
    Rows (a)-(c) use the same randomized environment and initial states, but with different attitude policies:
    (a) the \texttt{SEAMLiS} \texttt{gatekeeper}, (b) visibility-promoting~(VP), and (c) constant yaw-rate.
    The left panels show the final exploration snapshots, using the same visualization convention as Fig.~\ref{fig:exploration_progression}.
    The right panels show the visibility constraint value $h_i^{\textup{vis}}(t)$ for the two robots during the rollout; the dashed black line marks $h_i^{\textup{vis}}=0$.
    Positive values indicate that the critical known-free/unknown boundary can be brought into the FoV while retaining the braking-distance margin in~\eqref{eq:hvis}, whereas negative values indicate violation of the visibility-safety condition.
    Segments plotted at $+\infty$ correspond to the inactive case $d_i^{\textup{c}}=+\infty$, where no critical point appears along the current look-ahead path.
    The \texttt{SEAMLiS} rollout keeps both robots at nonnegative visibility margin, while VP and constant yaw-rate repeatedly allow $h_i^{\textup{vis}}<0$; in the VP case, the final snapshot shows collision with an initially unknown obstacle.
    }
    \label{fig:hvis_rollout}
\end{figure*}

\begin{figure*}[t]
    \centering
    \includegraphics[width=0.98\textwidth]{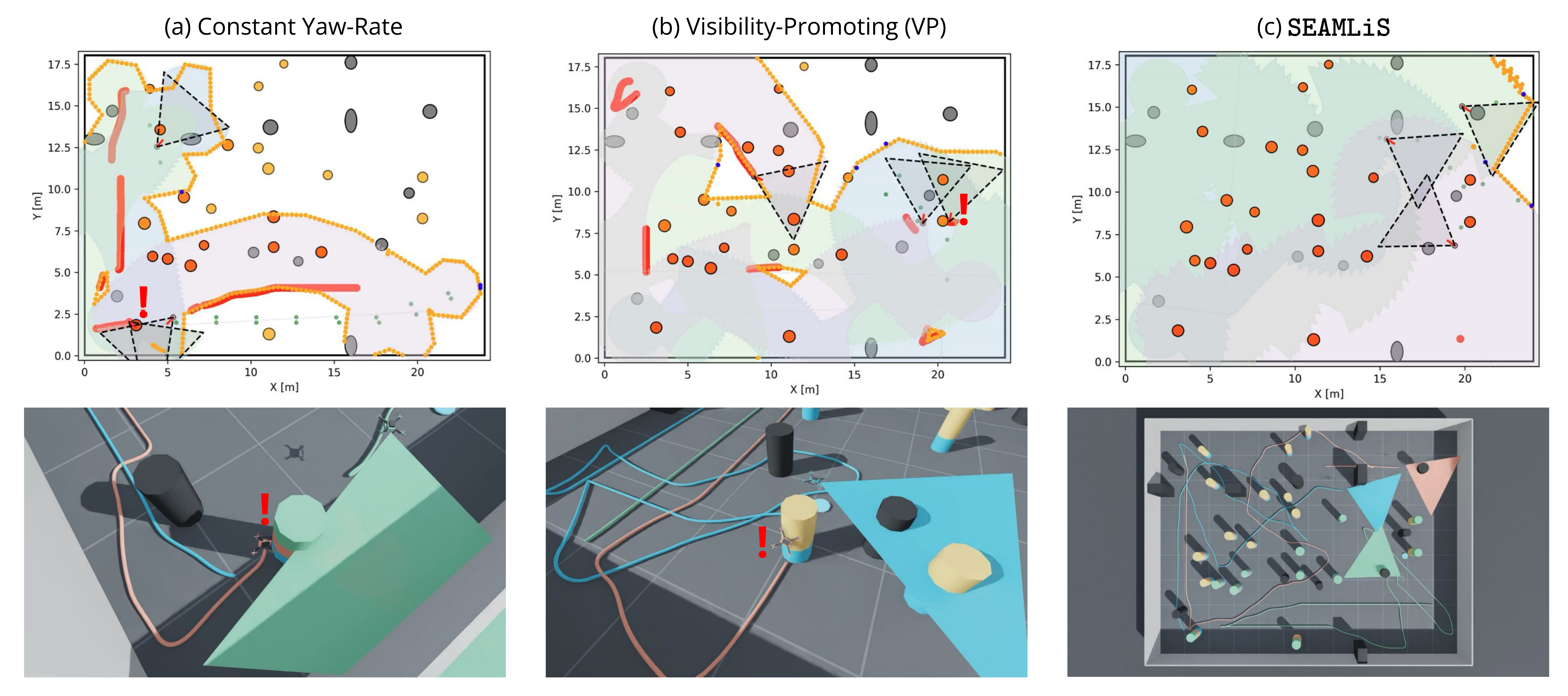}
    \caption{Representative Isaac Sim rollouts for frontier-based exploration with three robots. Initially unknown obstacles are rendered in light orange, and robot-colored obstacle regions indicate local detections by the corresponding robots. (a) Constant yaw-rate: the red robot collides with an obstacle after late detection. (b) Visibility-promoting (VP): the red robot collides with an initially unknown obstacle before detection. (c) \texttt{SEAMLiS}: the team safely explores and builds the map without any safety violation.}
    \label{fig:isaac_results}
\end{figure*}

\paragraph{Attitude-control baselines}
We compare four attitude configurations:
\begin{enumerate}[label=\roman*)]
    \item \textbf{Constant yaw-rate.} The robot executes a constant yaw-rate command $\omega_i=\omega_{\textup{const}}$.
    \item \textbf{Visibility-promoting (VP).} The robot always executes the nominal visibility-promoting policy $\pi_\theta^{\textup{vis}}$.
    \item \textbf{Velocity tracking.} The robot always executes the backup velocity-tracking yaw policy in~\eqref{eq:yaw_backup}.
    \item \textbf{\texttt{SEAMLiS}.} The robot uses the attitude \texttt{gatekeeper} filter in \autoref{alg:seamlis_gatekeeper}, with $\pi_\theta^{\textup{vis}}$ as the nominal policy and velocity tracking as the backup policy.
\end{enumerate}
All \texttt{gatekeeper} hyperparameters are fixed across trials.

\begin{figure*}[t]
    \centering
    \includegraphics[width=0.98\textwidth]{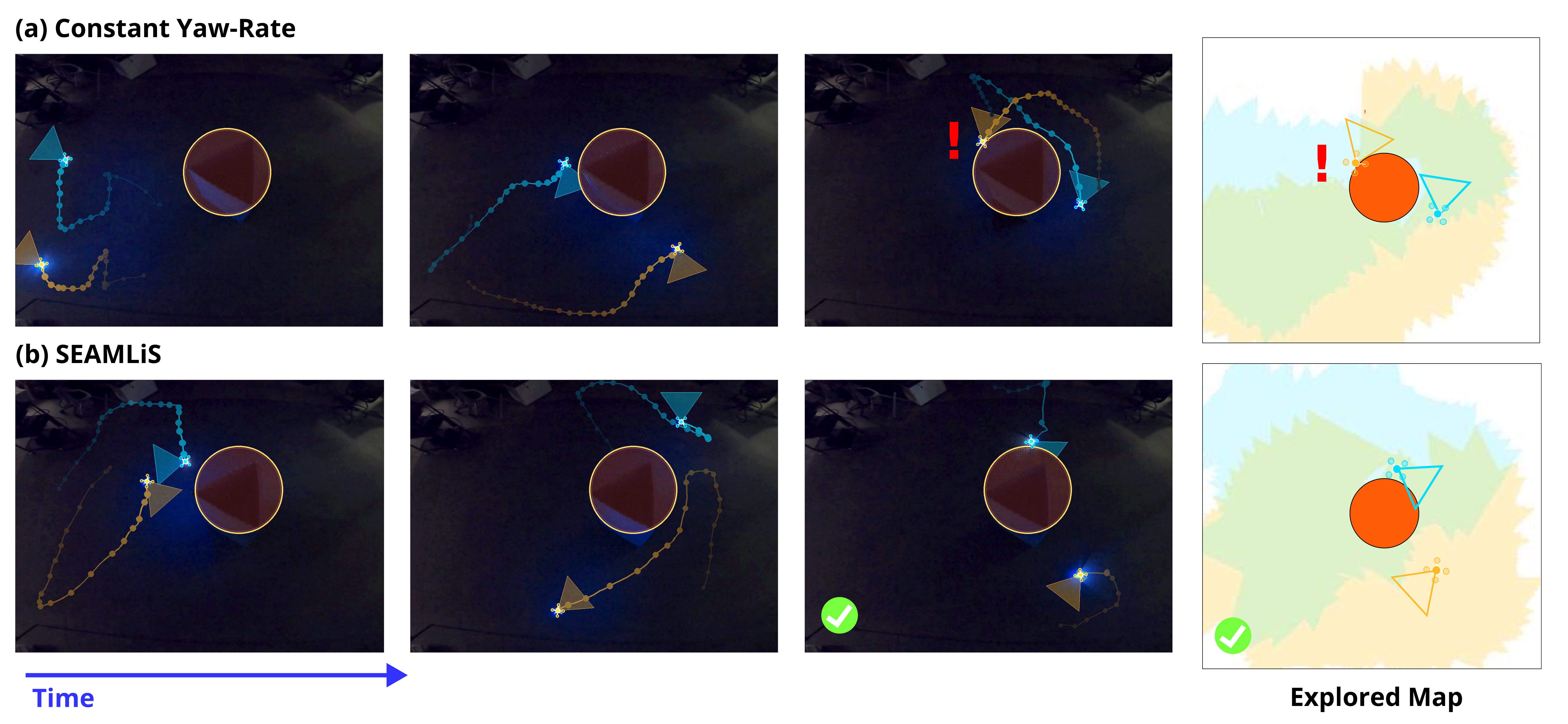}
    \caption{Hardware experiments with two Crazyflie quadrotors in a Vicon-tracked workspace. 
    Rows compare (a) the constant yaw-rate baseline and (b) \texttt{SEAMLiS}. 
    For each method, the first three columns show selected top-down video snapshots from the physical flight over time, and the rightmost column shows a post-processed exploration-map snapshot reconstructed from the corresponding ROS~2 bag logs. 
    The circular obstacle at the center is initially unknown to both robots and is inserted into each robot's local map only after it enters that robot's modeled finite-FoV. 
    Under constant yaw-rate, one robot collides with the obstacle even though the other robot has already locally detected it. 
    Under \texttt{SEAMLiS}, the attitude \texttt{gatekeeper} maintains motion-relevant visibility, allowing the MPC-CBF controller to avoid the obstacle.}
    \label{fig:hardware_results}
\end{figure*}

\paragraph{Metrics}
Each configuration is evaluated over $100$ randomized trials. We set the evaluation timeout to $160\,\mathrm{s}$, $80\,\mathrm{s}$, and $60\,\mathrm{s}$ for $R=1$, $R=2$, and $R=3$, respectively. We report: (i) \textbf{success rate}, the fraction of trials in which the robot team reaches $98\%$ free-space coverage before the timeout without collision; (ii) \textbf{collision rate}, the fraction of trials that terminate due to collision; (iii) \textbf{mean number of unknown-region overlap violations per robot}, where a violation is counted whenever $\operatorname{Robot}(\vp_i(t)) \cap \calU_i(t) \neq \emptyset$; and (iv) \textbf{exploration time}, averaged over successful trials only. Success and collision rates do not necessarily sum to one because some trials time out without collision.

\paragraph{Implementation and compute}
All simulations are implemented in Python and run on an Apple MacBook with an M4 CPU. The exploration environment is a $24\,\mathrm{m} \times 18\,\mathrm{m}$ workspace, and success is declared when the accumulated sensed free-space area reaches $98\%$ of the true free space without collision. Frontier extraction and coverage computation use geometric set operations on the accumulated FoV footprints. Local paths are planned with A* over each robot's current map, treating known obstacles as occupied and unknown cells optimistically as collision free.

The positional controller is shared across all attitude policies. We use a finite-horizon MPC-CBF controller with horizon $10$ steps and linear class-$\calK$ functions $\alpha_{0}=\alpha_{1}=0.55$, enforcing CBF constraints for the nearby known obstacles, locally detected hidden obstacles, and neighboring robots modeled with their communicated pose and estimated velocity. Hidden obstacles are inserted only into the detecting robot's local constraint set; obstacle lists or occupancy grids are not shared across robots. For \texttt{SEAMLiS}, the attitude \texttt{gatekeeper} uses the visibility-promoting yaw policy as the nominal controller and velocity-tracking yaw as the backup controller, with nominal horizon $1.3$ s, backup horizon $1.2$ s, and horizon decrement $0.05$ s. Batch trials are deterministic from a master random seed and are parallelized across independent CPU workers. The code and experiment scripts are available online.\footnote{Public repository: \href{https://github.com/tkkim-robot/seamlis}{https://github.com/tkkim-robot/seamlis}}

\subsection{Results}\label{subsec:sim_results}

\autoref{tab:sim_benchmark} summarizes the results for both goal assignment modules, across $R\in\{1,2,3\}$ robots and the four attitude policies. Figure~\ref{fig:exploration_progression} shows representative snapshots of \texttt{SEAMLiS} executions in the one-, two-, and three-robot settings. These snapshots visualize robot trajectories, finite-FoV sensing footprints, frontier points, local A* waypoints, accumulated known-free regions, and hidden-obstacle detections over time. Representative baseline failure cases are shown in Fig.~\ref{fig:baseline_failures}, which provides qualitative evidence for the collision trends reported in Table~\ref{tab:sim_benchmark}.

\textit{Safety:}
Across all tested team sizes and both goal-assignment modules, \texttt{SEAMLiS} succeeds in all reported trials with zero collisions. In contrast, the constant yaw-rate and VP baselines exhibit collisions in several multi-robot settings; for example, constant yaw-rate reaches a collision rate of $1.00$ for frontier-based exploration with $R=2,3$ and for the CoScan-inspired allocator with $R=3$, while VP reaches collision rates of $0.47$ and $1.00$ for frontier-based exploration with $R=2$ and $R=3$, respectively. The velocity-tracking baseline is also collision-free in all reported settings, but it never completes the exploration task within the robot-count-dependent timeout. These results show that motion-aligned sensing alone is safe but overly conservative, whereas \texttt{SEAMLiS} combines collision avoidance with timely exploration completion.

The representative failures in Fig.~\ref{fig:baseline_failures} show two mechanisms behind these collisions. In Fig.~\ref{fig:baseline_failures}(a), the VP policy rotates the sensor toward information-rich regions rather than maintaining sufficient visibility along the direction of motion, and the robot collides with a hidden obstacle after insufficient look-ahead. In Fig.~\ref{fig:baseline_failures}(b), the green robot moves into a region where the blue robot has already detected an obstacle; however, because the robots do not share full occupancy maps or obstacle lists, the obstacle is absent from the green robot's local CBF constraint set. This illustrates why limited sensing and limited information sharing jointly make perception-limited safety failures more prominent in decentralized exploration.

Figure~\ref{fig:hvis_rollout} further illustrates the mechanism behind these safety outcomes by directly plotting the visibility margin $h_i^{\textup{vis}}(t)$ in a representative two-robot rollout. With the \texttt{SEAMLiS}' \texttt{gatekeeper}, both robots remain in the visibility-safe region $h_i^{\textup{vis}}\ge0$, whereas the VP and constant yaw-rate policies repeatedly produce negative margins. These negative intervals indicate loss of the certified visibility margin near the known-free/unknown boundary, explaining why hidden obstacles can be detected too late for the positional CBF controller to avoid collision.

\textit{Unknown-region overlap violations:} 
The unknown-region overlap violation metric measures empirical exposure to locally unknown space before it has been sufficiently observed. The velocity-tracking baseline yields zero unknown-region overlap violations because it keeps the sensor aligned with the direction of motion, but this conservative behavior reduces information gain and prevents timely completion. The constant yaw-rate and VP baselines can explore efficiently in some settings, but their unknown-region overlap violations and collision rates indicate repeated exposure to hidden obstacles. \texttt{SEAMLiS} keeps the unknown-region overlap violation count low while completing all reported trials without collisions, which supports the intended behavior of the attitude \texttt{gatekeeper}: nominal visibility maximization is used when certified safe, and motion-aligned backup sensing is invoked near critical unknown boundaries. The remaining nonzero counts arise from discrete simulation updates and the simplified sensed-map representation used for computational efficiency in simulation. The visualization in Fig.~\ref{fig:baseline_failures}(b) further emphasizes that a hidden obstacle detected by one robot is not necessarily available to another robot unless it has also been locally sensed.

\textit{Exploration efficiency:}
\texttt{SEAMLiS} retains much of the efficiency of the VP policy while eliminating the observed collisions. In single-robot cases, VP can be slightly faster because it never switches away from information-greedy yaw. In multi-robot cases, \texttt{SEAMLiS} remains comparable to VP in exploration time and is substantially more reliable, since the safety filter prevents late obstacle detections that would otherwise lead to collision. Compared with velocity tracking, \texttt{SEAMLiS} avoids the overly conservative behavior that prevents timely completion.

\subsection{Isaac Sim Validation}\label{subsec:isaac}
Representative Isaac Sim~\cite{nvidia_isaac_sim} rollouts are shown in Fig.~\ref{fig:isaac_results}. We ran Isaac Sim experiments for the same one-, two-, and three-robot settings as in the randomized benchmark, using the same task setup and controller hyperparameters. The examples in Fig.~\ref{fig:isaac_results} correspond to the three-robot frontier-based setting.

In the visualization, initially unknown obstacles are rendered in light orange. Local obstacle detections are indicated by robot-colored regions on each cylindrical obstacle: when robot $i$ detects an obstacle, its corresponding region is recolored using robot $i$'s color. This visualization is consistent with the decentralized setting, since detected obstacle information is not shared across robots.

The figure highlights the distinct behaviors of the compared methods. In Fig.~\ref{fig:isaac_results}(a), the constant yaw-rate baseline causes the red robot to collide with an obstacle even after detection, showing that the obstacle was not detected early enough to preserve a feasible avoidance maneuver. In Fig.~\ref{fig:isaac_results}(b), the visibility-promoting baseline causes the red robot to collide with an unknown obstacle that had not yet been detected. In contrast, Fig.~\ref{fig:isaac_results}(c) shows that \texttt{SEAMLiS} safely completes the exploration task and builds the map without any safety violation. These Isaac Sim results are consistent with the trends observed in the randomized simulation benchmark and further support the practical effectiveness of the proposed safety filters.

\section{HARDWARE EXPERIMENTS}
\label{sec:hardware}

\subsection{Experimental Setup}

We further validate the proposed safety layer on two Crazyflie quadrotors in an indoor flight workspace with Vicon localization. The workspace is approximately $3.35\,\mathrm{m} \times 3.50\,\mathrm{m}$, and both robots fly at a fixed height of $z=1.0\,\mathrm{m}$. The planar motion is controlled using the double-integrator model with decoupled yaw in~\eqref{eq:dyn}. The robot radius is set to $r_{\textup{rob}}=0.075\,\mathrm{m}$.

The environment contains one initially unknown circular obstacle centered at $(0,0)$ in the Vicon frame with radius $0.3\,\mathrm{m}$. To focus the hardware test on decentralized execution and safety-filter behavior, the limited-FoV sensing event is computed geometrically from the Vicon state rather than from an onboard camera. The obstacle is therefore withheld from each robot's local map until it lies within that robot's modeled sensing footprint. This setup preserves the key limited-sensing mechanism studied in the paper while allowing baseline collisions with a physical obstacle to be observed in a controlled experiment.

Each Crazyflie is connected to a laptop through ROS~2~\cite{macenski_robot_2022}, and each robot runs its own exploration and control stack in an independent ROS~2 node. The robots exchange only their states and frontier information; they do not share maps, obstacle lists, planned paths, or control inputs. Thus, all goal assignment, path planning, sensing updates, and safety-filter computations are performed separately for each robot. The global goal assignment uses frontier-based exploration, and A* is used to generate a waypoint path toward the selected frontier on each robot's local map.

The hardware parameters are chosen as a workspace-scaled instantiation of the limited-sensing setting. We use $v_{\max}=1.35\,\mathrm{m/s}$, $a_{\max}=1.5\,\mathrm{m/s^2}$, $\theta_{\textup{fov}}=70^\circ$, and $l_{\textup{range}}=0.25\,\mathrm{m}$. The translational controller is the same MPC-CBF controller used in simulation, with horizon $10$ and class-$\calK$ gains $\alpha_0=\alpha_1=0.55$. For \texttt{SEAMLiS}, the attitude \texttt{gatekeeper} uses the visibility-promoting yaw policy as the nominal controller and velocity tracking as the backup controller, with nominal horizon $0.4\,\mathrm{s}$ and backup horizon $1.8\,\mathrm{s}$. We compare the constant yaw-rate baseline and \texttt{SEAMLiS}.

\subsection{Results}

Figure~\ref{fig:hardware_results} shows representative hardware executions for the constant yaw-rate baseline and \texttt{SEAMLiS}. Under constant yaw-rate, one robot collides with the initially unknown obstacle. The collision occurs even though the other robot has already locally sensed the obstacle, because obstacle information is not shared as a full map in the decentralized setting. Consequently, the obstacle is absent from the colliding robot's local CBF constraint set until that robot detects it with its own limited-FoV sensor model. The constant yaw-rate policy does not ensure that the direction of motion remains sufficiently visible, so the obstacle can be detected too late for the positional safety controller to avoid collision.

In contrast, \texttt{SEAMLiS} avoids the obstacle in the same hardware setting. When the robot approaches the boundary between known-free and unknown space, the attitude \texttt{gatekeeper} filters the nominal yaw command and invokes the velocity-tracking backup policy when needed. This keeps the motion-relevant unknown boundary visible early enough for the MPC-CBF controller to react after detection. The hardware result is consistent with the simulation and Isaac Sim trends: constant yaw-rate can fail under limited sensing and decentralized map information, whereas \texttt{SEAMLiS} preserves sufficient visibility for safe execution.

\section{CONCLUSION}

This paper presented \texttt{SEAMLiS}, a modular execution-layer safety framework for decentralized multi-robot exploration under limited sensing. The central idea is to separate exploration efficiency from safety enforcement: upstream goal assignment and planning modules can remain optimistic and information-driven, while \texttt{SEAMLiS} ensures that the robot observes critical unknown boundaries early enough to preserve feasible collision avoidance. The proposed attitude \texttt{gatekeeper} switches between visibility-promoting and velocity-tracking yaw policies, and the positional \ac{CBF} controller avoids known, newly detected, and inter-agent obstacles. Under stated sensing and feasibility conditions, the combined controller guarantees collision avoidance. 
Simulation results with frontier-based exploration and a decentralized CoScan allocator show that \texttt{SEAMLiS} achieves collision-free exploration across all tested team sizes while maintaining exploration times comparable to efficiency-oriented baselines. Hardware experiments with two Crazyflie quadrotors further show the same safety trend in physical flight: the constant yaw-rate baseline collides with an initially unknown obstacle, whereas \texttt{SEAMLiS} avoids the collision by preserving the visibility constraint.




\section*{Acknowledgment}

We would like to acknowledge the contribution of Elida Sensoy who worked on a preliminary simulation prototype related to this project, accessible to interested users in \url{https://github.com/elsensoy/seamlis-isaaclab}.






\bibliographystyle{IEEEtran}
\typeout{}
\bibliography{references.bib}

\end{document}